\documentclass[letterpaper]{article}
\usepackage{aaai2026}
\usepackage{times}
\usepackage{helvet}
\usepackage{courier}
\usepackage[hyphens]{url}
\usepackage{graphicx}
\usepackage[longnamesfirst]{natbib}
\usepackage{caption}
\usepackage{amsmath, amsthm, amsfonts, mathtools} 
\usepackage{graphicx}\usepackage{xcolor}
\usepackage{dsfont}
\usepackage{xspace}

\newtheorem{theorem}{Theorem}
\newtheorem{lemma}{Lemma}
\newtheorem{definition}{Definition}
\usepackage{algorithm}
\usepackage{algorithmic}
\usepackage{booktabs}
\usepackage{multirow}
\usepackage{graphicx}
\usepackage{amssymb}
\usepackage{subcaption}
\usepackage{newfloat}
\usepackage{listings}
\DeclareCaptionStyle{ruled}{labelfont=normalfont,labelsep=colon,strut=off}
\floatstyle{ruled}
\newfloat{listing}{tb}{lst}{}
\floatname{listing}{Listing}
\title{Theoretical and Empirical Analysis of Lehmer Codes to Search Permutation Spaces with Evolutionary Algorithms
}
\author{
    Yuxuan Ma\textsuperscript{\rm 1}, Valentino Santucci\textsuperscript{\rm 2}, Carsten Witt\textsuperscript{\rm 3}
}
\affiliations{
    \textsuperscript{\rm 1}Southern University of Science and Technology, China\\
    \textsuperscript{\rm 2}University for Foreigners of Perugia, Italy\\
    \textsuperscript{\rm 3}Technical University of Denmark, Denmark\\
    12212608@mail.sustech.edu.cn,valentino.santucci@unistrapg.it,cawi@dtu.dk
}

\begin{document}
\nocopyright
\maketitle

\begin{abstract}
A suitable choice of the representation of candidate solutions is crucial for the efficiency of evolutionary algorithms and related metaheuristics. We focus on problems in permutation spaces, which are at the core of numerous practical applications of such algorithms, e.\,g.\ in scheduling and transportation. Inversion vectors (also called Lehmer codes) are an alternative representation of the permutation space~$S_n$ compared to the classical encoding as a  vector of $n$~unique entries. In particular, they do not require any constraint handling. Using rigorous mathematical runtime analyses, we compare the efficiency of inversion vector encodings to the classical representation and give theory-guided advice on their choice. Moreover, we link the effect of local changes in the inversion code space to classical measures on permutations like the number of inversions. Finally, through experimental studies 
on linear ordering and quadratic assignment problems, we demonstrate the practical efficiency 
of inversion vector encodings. 
\end{abstract}
\section{Introduction}

Permutations are fundamental algebraic objects with broad applications, as they can model diverse concepts such as orderings or rankings of items, bijective mappings between two 
sets of items, and tours or sets of cycles within a collection of locations.
Because of their versatility, permutations form the solution space for many combinatorial optimization problems.
These problems, often called \textit{permutation problems}~\cite{santucci2025use}, include notable examples such as:
the Hamiltonian Path Problem (HPP), which involves finding the shortest path visiting every given location exactly once; 
the Linear Ordering Problem (LOP), which seeks for a maximum-weight directed acyclic subgraph within a given 
digraph;
and the Quadratic Assignment Problem (QAP), which aims to determine a cost-minimizing assignment of facilities to locations.
Permutation problems find applications in many tasks of practical relevance.
Among others, the LOP has been applied to model consensus ranking in computational social choice~\cite{marti2022exact} and to improve natural language translations~\cite{tromble2009learning}, while the QAP has been used for graph matching tasks~\cite{wang2021neural}. Therefore, there is 
interest
in developing effective and efficient solvers for permutation problems.

Nonetheless, permutation problems pose significant challenges.
First, they are often NP-hard, as are all the aforementioned examples.
Second, even among NP-hard problems, they are particularly challenging because the size of their search space increases factorially with the instance size---i.e., even faster than any exponential function.
This inherent difficulty has motivated the adoption of non-traditional algorithmic approaches, including metaheuristics, evolutionary algorithms (EAs) and, more recently, also neural networks and deep learning methodologies.

However, the constrained nature of permutations poses representation challenges for these methods.
The most common encoding scheme is the classical linear encoding, which represents a permutation as a vector of unique items satisfying mutual exclusivity.
This constraint is typically handled in EAs and other metaheuristics through specifically designed neighborhood structures~\cite{schiavinotto2007review} and variation operators~\cite{santucci2015algebraic}.
An alternative encoding is the permutation matrix, which represents a permutation as a binary matrix with exactly one $1$-entry per row and column.
This encoding is particularly popular in neural network architectures that generate permutations, where the final layer applies a temperature-decreasing Sinkhorn operator~\cite{mena2018learning} to iteratively transform the activation values of $n^2$ neurons into a permutation matrix.

To overcome the mutual exclusivity constraint, \textit{random key} representations were introduced \citep{bean1994genetic}.
In this encoding scheme, a real-valued vector can be decoded to a permutation by either argsorting its elements or replacing its values with their ranks within the vector.
However, a single permutation may correspond to an infinite number of vectors, leading to redundancy.
Additional criticisms to this approach are discussed by~\citet{santucci2020algebraic}.
Another method to get rid of mutual exclusivity is to reformulate the permutation problem in the parameter space of a differentiable probability model over permutations~\cite{CeberioSantucciTELO2023}. However, this approach typically requires more sophisticated algorithms.

Interestingly, there exists a less commonly used encoding scheme for permutations, called \textit{Lehmer code} or \textit{inversion vector}~\cite{lehmer1960teaching}, that: (i) does not require to satisfy any constraints, aside from trivial domain bounds,
 (ii) is in one-to-one correspondence with permutations, and
(iii) does not necessarily require sophisticated algorithms to work with.
Given a permutation, its Lehmer code is a vector of integers, where each entry indicates how many smaller items appear after that position in the original linear permutation.

This work investigates the effectiveness of Lehmer codes in EAs, with a primary emphasis on theoretical analysis. 
Our main theoretical contributions are threefold: we define a series of benchmark problems in the Lehmer code space that are closely related to current work in runtime analysis; we introduce simple algorithms for searching within this space; and, most importantly, we derive runtime bounds for these algorithms on the proposed benchmarks, contrasting them with existing algorithms and benchmarks defined over the traditional permutation space. Along the way, we improve an existing runtime analysis by~\citet{DoerrPohl} for a related, multi-valued search space by a factor of almost $\Theta(n^2)$. 
In addition, we also study the structural connections between Lehmer space and permutation space.
To the best of our knowledge, this is the first study to explore these directions.
To complement this, we also conduct an empirical evaluation on real-world instances of two different types of NP-hard permutation problems: the LOP, 
an ordering problem, and the QAP, an assignment problem. The source code is available at https://github.com/TrendMYX/LehmerEA.

The detailed mathematical proofs of our theoretical results are collected in a technical appendix. Moreover, we include an appendix with detailed experimental results.

\subsection{Related Work}
Lehmer code representations in EAs have mostly been studied  empirically so far~\cite{RegnierMcCallFactoradicPPSN14,marmion2015fitness,KhalilGAPermutation,UherLehmerTSP}, showing scenarios where they are preferable to classical encodings. 

The first work describing theoretical runtime analyses of EAs on permutation problems appeared 20 years ago~\citep{Scharnow2005}. 
However, since then, the major focus of runtime analyses of EAs 
has been
on binary search spaces (see~\citealp{DoerrNeumannBook2020} for an overview) and results for permutations were only sporadic, see, e.\,g.,~\citet{GGLAlgo2019,BassinB20}. Only 
recently, there has been increasing momentum in this area, by~\citet{Doerr2023} suggesting a new framework for the runtime analysis and~\citet{BaumannRSGECCO24, 
BaumannRSGECCO25} analyzing permutation spaces  
arising in graph drawing problems. However, these studies all use the canonical representation of permutations as a vector 
of $n$ unique entries and to the best of our knowledge, there has been no runtime analysis for 
other representations of permutations so far. Our study of Lehmer codes fills this gap.

\section{Preliminaries}
\label{sec:preliminaries}

\subsection{Concepts and Notation}
Let $[a..b] := \{z \in \mathbb{Z} \mid a \leq z \leq b\}$ and \mbox{$[k]:=[0..k-1]$}. A mapping from $[1..n]$ to itself is called a \textit{permutation} (of $[1..n]$) if it is bijective. We denote by $S_n$ the set of all permutations of $[1..n]$. A common notation represents a permutation $\sigma \in S_n$ as a vector $(\sigma(1), \sigma(2), \dots, \sigma(n))$. The \textit{identity} permutation in $S_n$ is $(1,2,\dots,n)$.
For clarity, we assume that the vector representation of a permutation is 1-indexed from the left.

As in~\citet{Doerr2023}, we adopt the standard composition \( \circ \) of permutations: for \( \sigma, \tau \in S_n \), 
\( \tau \circ \sigma \) is defined by \( (\tau \circ \sigma)(i) = \tau(\sigma(i)) \) for all \( i \in [1..n] \).

Another common notation is the \textit{cycle} notation, where a \textit{cycle} of length \( k \) (a \( k \)-cycle) is a permutation \( \sigma \in S_n \) such that there exist \( k \) distinct elements \( i_1, \dots, i_k \in [1..n] \) satisfying: for all \( j \in [1..k-1] \), \( \sigma(i_j) = i_{j+1} \), and \( \sigma(i_k) = i_1 \); while for all \( j \in [1..n] \setminus \{i_1, \dots, i_k\} \), \( \sigma(j) = j \). We denote such a cycle by \( \sigma = (i_1\ i_2\ \cdots\ i_k) \). Any permutation can be written as a product of disjoint cycles. A 2-cycle
is called a \textit{transposition}. A transposition of the form \( s_i = (i\ i{+}1) \) is called an \textit{adjacent transposition}.

Note that an adjacent transposition \( s_i = (i\ i{+}1) \) acts as an \textit{adjacent swap} of the values at indices \( i \) and \( i+1 \) of a permutation \( \sigma \in S_n \) when \( \sigma \) is right-multiplied by \( s_i \), i.e., \( \sigma \circ s_i \). In~\citet{Scharnow2005}, the authors define $\text{jump}(i,j)$ as an operation that causes the element at position $i$ to jump to position $j$ while the elements at positions $i+1,\dots,j$  (if $j>i$) or $j,\dots,i-1$ (if $j<i$) are shifted in the appropriate direction. We keep the same notation in this paper, although this operation is also referred to as \textit{insertion} in other works, such as \citet{baioletti2020variable}.

Consider 
$\sigma\in S_n$, then a pair of indices $(i,j)$ with $i<j$ and $\sigma(i)>\sigma(j)$ is called an \textit{inversion} of $\sigma$. The \textit{Lehmer Code} is an alternative way of encoding permutations and it utilizes the concept of inversions, which is why it is also known as the \textit{inversion vector}. The Lehmer code and the Lehmer code space can be defined in various ways. In this paper, we adopt the following definition.
\begin{definition}\label{def:lehmer}
The Lehmer code space $L_n$ is defined as the Cartesian product $[n]\times[n-1]\times\cdots\times[1]$, and the Lehmer code of any permutation $\sigma\in S_n$ is defined as a sequence of length $n$, denoted by \mbox{$L(\sigma):=(L(\sigma)_n,L(\sigma)_{n-1},\dots,L(\sigma)_1)$} where 
\begin{displaymath}
    L(\sigma)_{n-i+1}:=\#\{j>i\mid \sigma(j)<\sigma(i)\},i\in[1..n].
\end{displaymath}
\end{definition}

For example, consider a $\sigma=(3,5,4,1,2)\in S_5$. 
Its corresponding Lehmer code
is
$L(\sigma)=(2,3,2,0,0)$. By convention, the Lehmer code is indexed from $n$ down to $1$. The benefit is that position $i$ has cardinality $i$. In the remainder of this paper, we may omit the final entry (index $1$), as it is always zero by definition. There exists a one-to-one correspondence between the permutation space $S_n$ and the Lehmer code space $L_n$.

Some other useful definitions are as follows. The $n$-th Harmonic number is $H_n=\sum_{i=1}^n 1/i$, for $n=1,2,\dots$, and $H_0=0$. Given an event $A$, the indicator variable $\mathbf{1}_{\{A\}}$ equals $1$ if $A$ occurs and $0$ otherwise.
\subsection{Mathematical Analysis Tools}
Drift theory is about bounding the expected optimization time of a randomized search heuristic by mapping its state space to real numbers and considering the ``drift'', which is the expected decrease in distance from the optimum in a single step. 
In particular, we use multiplicative and variable drift (see \citealp{Lengler2020} for an overview of drift theory). Moreover, we use results on coupon collector processes with non-uniform
selection probabilities. 

\section{Algorithms and Benchmark Functions}
\label{sec:alg_and_benchmark}
Runtime analyses of EAs usually start
by considering simple benchmark algorithms on simple, well-structured benchmark functions. Such studies have been successfully performed for the search space $\{0,1\}^n$ 
of binary representations and enhanced our understanding 
of increasingly complex EAs \citep{DoerrNeumannBook2020}. 
We adopt this approach by adjusting the commonly studied randomized local search (RLS) and $(1+1)$-EA from binary spaces   
to the Lehmer code search space as given by Algs.~\ref{alg:RLS} and~\ref{alg:(1+1)EA-lehmer)}. Both algorithms can be equipped with different \textit{step operators} that decide how the individual mutates at a given position. While the RLS makes a step in exactly one position, the \mbox{$(1+1)$-EA} can make a step in each position $i\in[2..n]$, but with probability $1/(n-1)$. 

One may notice that the Lehmer code space is similar to the multi-valued decision variable search space $[r]^{n}$ considered in theoretical runtime analyses before \citep{Doerr2018}, but with decreasing domain sizes across dimensions. Consequently, we consider the following two \textit{step operators}, adapted from the above-mentioned work:
\begin{enumerate}
    \item The \textit{uniform} step operator: for $i\in[2..n]$, $\text{step}_i(x_i)$ chooses a value $[0..i-1]\backslash\{x_i\}$ uniformly at random.
    \item The $\pm 1$ step operator: for $i\in[2..n]$, $\text{step}_i(x_i)=\max\{0,x_i-1\}$ with probability $1/2$, and $\text{step}_i(x_i)=\min\{i-1,x_i+1\}$ otherwise.
\end{enumerate}
We refer to RLS equipped with the uniform step operator as RLS with uniform mutation strength, and RLS using the $\pm1$ step operator as RLS with unit mutation strength. The same naming convention applies to the $(1+1)$-EA. Given that domain sizes differ across dimensions in the Lehmer code space, we assign a probability vector $\mathbf{p}$ to guide the selection of the position to be modified. We consider the following two types of probability vectors:
\begin{enumerate}
    \item \textit{Uniform probability vector}: Each position $i\in[2..n]$ is selected with the same probability, 
    i.e.,
    $\textbf{p}_i=1/(n-1)$.
    \item \textit{Proportional probability vector}: The probability of selection for each position $i\in[2..n]$ depends on the 
    domain size at position $i$,
    specifically $\textbf{p}_i=2(i-1)/(n(n-1))$.
\end{enumerate}
\begin{algorithm}[tb]
\caption{$\text{RLS}$ minimizing a function $f: L_n\rightarrow \mathbb{R}$ with a given probability vector $\textbf{p}$ and step operator}
\label{alg:RLS}
\begin{algorithmic}[1]
\STATE Sample $x\in L_n$ uniformly at random
\FOR{$t=1,2,3,\dots$}
\STATE Choose $i\in[2..n]$ according to probability $\textbf{p}_i$
\STATE $y\gets x$
\STATE $y_i\gets\text{step}_i(x_i)$
\STATE \textbf{if} $f(y)\leq f(x)$ \textbf{then} $x\gets y$
\ENDFOR
\end{algorithmic}
\end{algorithm}

\begin{algorithm}[tb]
\caption{$(1{+}1)$-EA minimizing a function $f: L_n\rightarrow \mathbb{R}$ with a given step operator}
\label{alg:(1+1)EA-lehmer)}
\begin{algorithmic}[1]
\STATE Sample $x\in L_n$ uniformly at random
\FOR{$t=1,2,3,\dots$}
\FOR{$i$ from $n$ down to $2$}
\STATE Set $y_i{\gets}\text{step}_i(x_i)$ with probability $1/(n{-}1)$ and set $y_i\gets x_i$ otherwise
\ENDFOR
\STATE \textbf{if} $f(y)\leq f(x)$ \textbf{then} $x\gets y$
\ENDFOR
\end{algorithmic}
\end{algorithm}

\begin{algorithm}[t]
\caption{permutation-based $(1{+}1)$-EA minimizing a function $f: S_n\rightarrow \mathbb{R}$ with given mutation operator $\text{mut}(\cdot)$}
\label{alg:(1+1)EA-permutation)}
\begin{algorithmic}[1]
\STATE Sample $\sigma\in S_n$ uniformly at random
\FOR{$t=1,2,3,\dots$}
\STATE Choose $k\sim\text{Poi}(1)$
\STATE $\sigma^\prime\gets \text{mut}(\sigma,k)$
\STATE \textbf{if} $f(\sigma^\prime)\leq f(\sigma)$ \textbf{then} $\sigma\gets \sigma^\prime$
\ENDFOR
\end{algorithmic}
\end{algorithm}

We consider a permutation-based $(1+1)$-EA (see Alg.~\ref{alg:(1+1)EA-permutation)}) that operates on the classical linear representation of permutations, which we will refer simply as $S_n$ from now on. The mutation operator is parameterized by an integer $k$, and is defined via the composition of $k$ elementary operations which are randomly sampled. We investigate the following three mutation schemes:
\begin{enumerate}
    \item \textit{Transposition}: 
    $k$ transpositions $T_1,T_2,\dots,T_k$ 
    are selected
    independently and uniformly at random, and 
    $\text{mut}(\sigma,k)=\sigma\circ T_1\circ\cdots\circ T_{k-1}\circ T_k$.
    This operator has been studied in~\citet{Doerr2023}.
    \item \textit{Adjacent Swap}: 
    the algorithm selects $k$ adjacent transpositions $s_{i_1},s_{i_2},\dots,s_{i_k}$ independently and uniformly at random, and $\text{mut}(\sigma,k)=\sigma\circ s_{i_1}\circ s_{i_2}\circ\cdots\circ s_{i_k}$.
    \item \textit{Insertion}:
    $k$ insertions 
    $\text{jump}(i_1,j_1),\dots,\text{jump}(i_k,j_k)$
    are selected
    independently and uniformly at random, and $\text{mut}(\sigma,k)=\tau$, where $\tau$ is 
    generated by sequentially applying the $k$ insertions on $\sigma$. This operator was 
    also studied
    in~\citet{Scharnow2005}.
\end{enumerate}

    The theoretical benchmark functions that we will analyze the algorithms on 
    are based on well-structured benchmarks from the literature, including \textsc{OneMax}, \textsc{LeadingOnes} and \textsc{BinaryValue} (\textsc{BinVal} for short)  for binary spaces \citep{DJWoneone} and the fitness functions derived in  \citet{Scharnow2005}. The simple structure but distinct fitness landscape of these functions illustrates typical optimization scenarios in evolutionary algorithms. Moreover, the functions have been used as building blocks in the design and analysis of more advanced theory-inspired benchmarks in multimodal optimization \citep{LissovoiOW23AIJ23}.
    
    We study three groups of benchmark functions based on the above-mentioned \textsc{OneMax}, \textsc{LeadingOnes} and \textsc{BinVal} to conduct the theoretical runtime analyses. 
    \begin{enumerate}
        \item $\mathcal{L}\textsc{-OneMax}$ and $\textsc{INV}$: $\mathcal{L}$\textsc{-OneMax} is a unimodal linear function defined over $L_n$ where for all $l\in L_n$, $\mathcal{L}\textsc{-OneMax}(l)=\sum_{i=2}^n l_i$. The $\textsc{INV}$ is defined over $S_n$ and for all $\sigma\in S_n$, $\textsc{INV}(\sigma)$ counts the number of inversions in $\sigma$. 
        Our goal is to \textit{minimize} both functions, and the optimum point is $(0,0,\dots,0)$ and the identity permutation, respectively.
        
        \item $\mathcal{L}\textsc{-LeadingZeros}$ and $\textsc{PLeadingOnes}$: the two functions are defined over $L_n$ and $S_n$, respectively. For all $l\in L_n$, 
        \begin{align*}
            &\mathcal{L}\textsc{-LeadingZeros}(l)\\
            &=\max\{i\in[0..n]\mid \forall j\in[n-i+1..n],l_j=0\}, 
        \end{align*}
        while for all $\sigma\in S_n$, 
        \begin{align*}
        &\textsc{PLeadingOnes}(\sigma)\\
        &=\max\{i\in[0..n]\mid \forall j\in[1..i],\sigma(j)=j\}
        \end{align*}
        Our goal is to \textit{maximize} both functions, thus $f(y)\leq f(x)$ should be replaced by $f(y)\geq f(x)$ in all pseudo-code.
        
        \item $\textsc{FacVal}$, $\textsc{NVal}$ and $\textsc{LexVal}$: The $\textsc{FacVal}$ is defined over $L_n$ where for all $l\in L_n$,
        \begin{displaymath}
            \textsc{FacVal}(l)=\sum_{i=2}^n (i-1)!\cdot l_i
        \end{displaymath}
        the $\textsc{NVal}$ is defined over the muti-valued decision variables space $[n]^n$ and for all $x\in [n]^n$,
        \begin{displaymath}
            \textsc{NVal}(x)=\sum_{i=1}^n n^{i-1}\cdot x_i 
        \end{displaymath}
        the $\textsc{LexVal}$ is defined over $S_n$. In fact, all permutations $\sigma\in S_n$ can be ordered according to the lexicographic order of their vector representations, indexed from $0$. The $\textsc{LexVal}$ of a permutation $\sigma\in S_n$ is defined as its rank in this order, e.g., consider $\sigma=(1,3,2)\in S_3$, then $\textsc{LexVal}(\sigma)=1$. Our goal is to \textit{minimize} all three functions.
    \end{enumerate}

\section{Connection between $L_n$ and $S_n$}
\label{sec:connection}

To illustrate the relationships between a Lehmer vector $x \in L_n$ and its corresponding permutation $\sigma \in S_n$, it is crucial to note that the sum of the entries in $x$ equals the number of inversions in $\sigma$, which can be proved by definition.
Therefore, $\mathcal{L}\textsc{-OneMax}\circ L$ and $\textsc{INV}$ are equivalent functions where $L$ is the bijection from $S_n$ to $L_n$. The following helper lemma shows that $\mathcal{L}\textsc{-LeadingZeros}\circ L$ and $\textsc{PLeadingOnes}$, as well as $\textsc{FacVal}\circ L$ and $\textsc{LexVal}$, are equivalent functions. We denote the function $\mathcal{L}\textsc{-LeadingZeros}$ by $\textsc{LZ}$ and $\textsc{PLeadingOnes}$ by $\textsc{LO}$ for brevity.
\begin{lemma}
\label{lem:EquivalentFunction}
    For any $\sigma\in S_n$, let $L:S_n\rightarrow L_n$ denote the bijection which is defined before, then $\textsc{LO}(\sigma)=\textsc{LZ}(L(\sigma))$ and $\textsc{LexVal}(\sigma)=\textsc{FacVal}(L(\sigma))$.
\end{lemma}

Moreover, it is well known that applying an adjacent swap to a permutation $\sigma$ adds or removes exactly one inversion in the resulting permutation \citep{santucci2015algebraic}.
This makes it particularly interesting to examine how an adjacent swap affects the corresponding Lehmer code representation. 
The following lemma shows that an adjacent swap affects two entries of the Lehmer code by swapping their entries and changing one of them by $\pm 1$.
\begin{lemma}
\label{lem:AdjacentLehmerCode-2}
    Given $\sigma,\tau\in S_n$ where $\tau_i=\sigma_{i+1},\tau_{i+1}=\sigma_i$ and $\tau_j=\sigma_j$ for all $j\in[1..n]\backslash\{i,i+1\}$. Let $A$ denote the event that $L(\sigma)_{n-i+1}\leq L(\sigma)_{n-i}$. Then we have,
    \begin{itemize}
        \item 
        $L(\tau)_{n-i+1} = L(\sigma)_{n-i}+\mathbf{1}_{A}$,
        
        \item
        $L(\tau)_{n-i} = L(\sigma)_{n-i+1}-\mathbf{1}_{\overline{A}}$, and
        
        \item 
        $L(\tau)_{n-j+1}=L(\sigma)_{n-j+1}$ for all $j\in[1..n]\backslash\{i,i+1\}$.
    \end{itemize}
\end{lemma}

Adjacent swaps are of particular interest because they are a proper subset of both (generic) swaps and insertions.
Moreover, they allow recreating both insertions and generic swaps as follows.
An insertion $\text{jump}(i,j)$ is the composition the adjacent transpositions $s_i \circ s_{i+1} \circ \cdots \circ s_{j-1}$ for $j>i$,
while a transposition $(i\ j)$ equals $\text{jump}(i,j) \circ \text{jump}(j-1,i)$ for $j>i$.
Analogous expressions hold for $j<i$.

\section{Runtime Analysis in Lehmer Code Space}
\label{sec:runtime}
In this section, we rigorously analyze the performance of RLS and the $(1+1)$-EA on several representative benchmark functions introduced in Sect.~\ref{sec:alg_and_benchmark}. These analyses provide insights into the behavior of the algorithms under various conditions. In particular, we focus on deriving bounds on the expected number of function evaluations required to reach an optimal solution, also referred to as the expected optimization time. 

\subsection{RLS with uniform mutation strength}
We begin by analyzing $\textsc{RLS}$ equipped with the uniform step operator. For the fitness functions $\mathcal{L}\textsc{-OneMax}$ or $\textsc{FacVal}$, obtaining tight bounds via drift analysis is challenging due to the decreasing dimension sizes in the Lehmer code space. Fortunately, the problem can be reformulated as a variant of the Coupon Collector’s Problem with unequal probabilities. The result is stated in the following theorem.
\begin{theorem}
    \label{thm:RLS-uniform-OneMax}
    Consider $\textsc{RLS}$ equipped with uniform step operator and uniform probability vector, then its expected optimization time minimizing either $\mathcal{L}\textsc{-OneMax}$ or $\textsc{FacVal}$ is bounded from above by $(n-1)^2\ln n+(n-1)^2$ and from below by $(n-1)^2\ln n-o(n^2\log n)$.
\end{theorem}
The following theorem provides an exact expression for the expected optimization time of RLS on $\mathcal{L}\textsc{-LeadingZeros}$. 
\begin{theorem}
\label{thm:RLS-uniform-LeadingZeros}
    The expected optimization time of \textsc{RLS}, when using the uniform step operator and uniform probability vector on the $\mathcal{L}\textsc{-LeadingZeros}$ function, is \mbox{$n^3/2-2n^2+nH_n+3n/2-H_n$}.
\end{theorem}

One may notice that there exists an imbalance in the optimization time of the operator across different positions due to the unequal domain sizes in the Lehmer code space. Therefore, it is natural to investigate whether the expected optimization time of $\textsc{RLS}$ will be faster if equipped with the proportional probability vector which, from the probability perspective, smooths out the effects caused by differences in domain sizes. The following two theorems state the results.
\begin{theorem}
\label{thm:RLS-uniform-proportional-OneMax}
    Consider $\textsc{RLS}$ equipped with uniform step operator and proportional probability vector, then its expected optimization time minimizing either $\mathcal{L}\textsc{-OneMax}$ or $\textsc{FacVal}$ is $n(n-1)(\ln(n)+\Theta(1))/2$.
\end{theorem}
\begin{theorem}
\label{thm:RLS-uniform-proportional-LeadingZeros}
    Consider $\textsc{RLS}$ equipped with uniform step operator and proportional probability vector, then its expected optimization time maximizing $\mathcal{L}\textsc{-LeadingZeros}$ is $n^3/2-n^2 H_n/2-n^2/2+nH_n/2$.
\end{theorem}
We note that the proportional probability vector gives no 
asymptotic speed-up compared to the uniform one. 
Moreover, for the unit step operator, the mutation strength is the same for each position. Therefore, the proportional probability vector will no longer be considered in the subsequent analyses.
\subsection{RLS with unit mutation strength}
\begin{theorem}
\label{thm:RLS-unit-OneMax}
    Consider $\textsc{RLS}$ equipped with $\pm1$ step operator and uniform probability vector, then its expected optimization time minimizing either $\mathcal{L}\textsc{-OneMax}$ or $\textsc{FacVal}$ is $\Theta(n^2)$.
\end{theorem}
\begin{theorem}
\label{thm:RLS-unit-LeadingZeros}
    Consider $\textsc{RLS}$ equipped with $\pm1$ step operator and uniform probability vector, then its expected optimization time maximizing $\mathcal{L}\textsc{-LeadingZeros}$ is \mbox{$2n^4/9-7n^3/18+n^2/9+n/18$}.
\end{theorem}
\subsection{$(1+1)$-EA with uniform mutation strength}
Now, we begin to analyze $(1+1)$-EA equipped with uniform step operator, and the following theorem provides a general lower bound for its expected optimization time on both $\mathcal{L}\textsc{-OneMax}$ and $\textsc{FacVal}$.
\begin{theorem}
\label{thm:linear-function-lower-bound}
    Consider $(1+1)$-EA equipped with uniform step operator, then its expected optimization time minimizing either $\mathcal{L}\textsc{-OneMax}$ or $\textsc{FacVal}$ is bounded from below by $\Omega(n^2\log n)$.
\end{theorem}
We state the following two theorems that provide the upper bounds for the expected optimization time on $\mathcal{L}\textsc{-OneMax}$ and $\textsc{FacVal}$ which match their lower bound asymptotically.
\begin{theorem}
\label{thm:(1+1)-EA-uniform-OneMax}
    Consider $(1+1)$-EA equipped with uniform step operator, then its expected optimization time minimizing $\mathcal{L}\textsc{-OneMax}$ is bounded from above by \mbox{$\mathrm{e}(n-1)^2\ln n+2\mathrm{e}(n-1)^2-2\mathrm{e}(n-1)$}.
\end{theorem}
\begin{theorem}
\label{thm:(1+1)-EA-uniform-FacVal}
    Consider $(1+1)$-EA equipped with uniform step operator, then its expected optimization time minimizing $\textsc{FacVal}$ is bounded from above by \mbox{$69.2(n-1)^2\ln n+\mathcal{O}(n^2)$}.
\end{theorem}

\citet{DoerrPohl} analyze $(1+1)$-EA on all linear functions in the search space $[r+1]^n$. Applying the bound they provide, the expected optimization time of \mbox{$(1+1)$-EA} on $\textsc{NVal}$ is bounded from above by $\mathcal{O}(n^4\log\log n)$ and from below by $\Omega(n^2\log n)$. We improve this bound via the same proof technique used in Theorem~\ref{thm:(1+1)-EA-uniform-FacVal}.
\begin{theorem}
\label{thm:(1+1)-EA-uniform-NVal}
    The expected optimization time of $(1+1)$-EA on $\textsc{NVal}$ is $\Theta(n^2\log n)$.
\end{theorem}
In the following theorem, we provide a tight bound for the expected optimization time of $(1+1)$-EA equipped with uniform step operator on $\mathcal{L}\textsc{-LeadingZeros}$.
\begin{theorem}
\label{thm:(1+1)-EA-uniform-LeadingZeros}
Consider $(1+1)$-EA equipped with uniform step operator, then its expected optimization time maximizing $\mathcal{L}\textsc{-LeadingZeros}$ is $(\mathrm{e}-2)(n-1)^3+(3-3\mathrm{e}/2)(n-1)^2+R_n$, where $R_n>0$ and $R_n=\mathcal{O}(n\log n)$.
\end{theorem}

\subsection{$(1+1)$-EA with unit mutation strength}
The following theorem provides a general lower bound for the expected optimization time of $(1+1)$-EA equipped with $\pm1$ step operator on both $\mathcal{L}\textsc{-OneMax}$ and $\textsc{FacVal}$.
\begin{theorem}
\label{thm:(1+1)-EA-unit-lower-bound}
    Consider $(1+1)$-EA equipped with $\pm1$ step operator, then its expected optimization time minimizing either $\mathcal{L}\textsc{-OneMax}$ or $\textsc{FacVal}$ is bounded from below by $(n-1)^2$.
\end{theorem}
The following two theorems respectively provide asymptotic upper bounds for $\mathcal{L}\textsc{-OneMax}$ and $\text{FacVal}$.
\begin{theorem}
\label{thm:(1+1)-EA-unit-OneMAx}
    Consider $(1+1)$-EA equipped with $\pm 1$ step operator, then its expected optimization time minimizing $\mathcal{L}\textsc{-OneMax}$ is bounded from above by $\mathcal{O}(n^2)$.
\end{theorem}
\begin{theorem}
\label{thm:(1+1)-EA-unit-FacVal}
    Consider $(1+1)$-EA equipped with $\pm1$ step operator, then its expected optimization time minimizing $\textsc{FacVal}$ is bounded from above by $\mathcal{O}(n^2\log n)$.
\end{theorem}
To conclude this section, we provide the exact expression of the expected optimization time on $\mathcal{L}\textsc{-LeadingZeros}$ in the following theorem.
\begin{theorem}
\label{thm:(1+1)-EA-unit-LeadingZeros}
    Consider $(1+1)$-EA equipped with $\pm1$ step operator, then its expected optimization time maximizing $\mathcal{L}\textsc{-LeadingZeros}$ is $\frac{32\sqrt{\mathrm{e}}-52}{3}(n-1)^4+\frac{28-16\sqrt{\mathrm{e}}}{3}(n-1)^3+\frac{13\sqrt{\mathrm{e}}-12}{36}(n-1)^2-\frac{\sqrt{\mathrm{e}}}{48}(n-1)-\Theta(1)$.
\end{theorem}
\subsection{Comparison of Runtime Results}
Here, we mainly compare the results obtained in this section
with previous results based on the canonical representation of permutations. In the seminal work by \citet{Scharnow2005}, the authors rigorously analyze the performance of a $(1+1)$-EA which chooses $s{+}1$ ($s{\sim}\text{Poi}(1)$) operators where each operator is decided by a fair coin flip to be either \textit{jump} or \textit{transposition}. The authors show that the expected optimization time on four fitness functions including $\textsc{INV}$ can be bounded by $\mathcal{O}(n^2\log n)$ and $\Omega(n^2)$. In the recent work by \citet{BaumannRSGECCO24}, the authors show that the expected time for $\textsc{RLS}$ with \textit{adjacent swap} to optimize $\textsc{INV}$ can be bounded by $\Theta(n^2)$. However, for the general $(1+1)$~EA using the classical representation on $\textsc{INV}$, the best available upper bound is still $\mathcal{O}(n^2\log n)$. For comparison, our bound for \mbox{$\mathcal{L}$-$\textsc{OneMax}$}   
in Theorem~\ref{thm:(1+1)-EA-unit-lower-bound} and Theorem~\ref{thm:(1+1)-EA-unit-OneMAx} is $\Theta(n^2)$ which is asymptotically tight.

In the work by \citet{Doerr2023}, the authors show that the expected time for $(1+1)$-EA with \textit{transposition} to optimize $\textsc{PLeadingOnes}$ can be bounded by $\Theta(n^3)$. The same polynomial also appears in 
Theorem~\ref{thm:(1+1)-EA-uniform-LeadingZeros} for the related $\mathcal{L}$-\textsc{LeadingZeros}, even with exact coefficients for the cubic and quadratic terms. Only for unit mutation (Theorem~\ref{thm:(1+1)-EA-unit-LeadingZeros}), the Lehmer representation gives a worse complexity, which is due to a random walk behavior.

Due to difficulty and complexity, analyses based on canonical representation are typically given in terms of asymptotic bounds. Therefore, to enable a more detailed comparison, we also include empirical analyses in in the following section.

\section{Experimental Analysis}\label{sec:exp}
\begin{figure*}[t]
  \centering
    \includegraphics[width=0.98\textwidth]{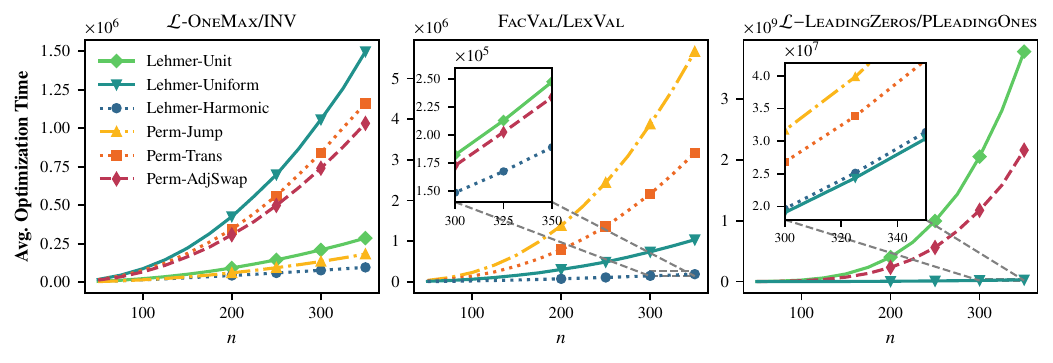}
    \caption{Summary of the results obtained in the experiments on theoretical benchmark functions.}
  \label{fig:comparison_all}
\end{figure*}

\begin{figure*}
    \centering
    \includegraphics[width=0.98\textwidth]{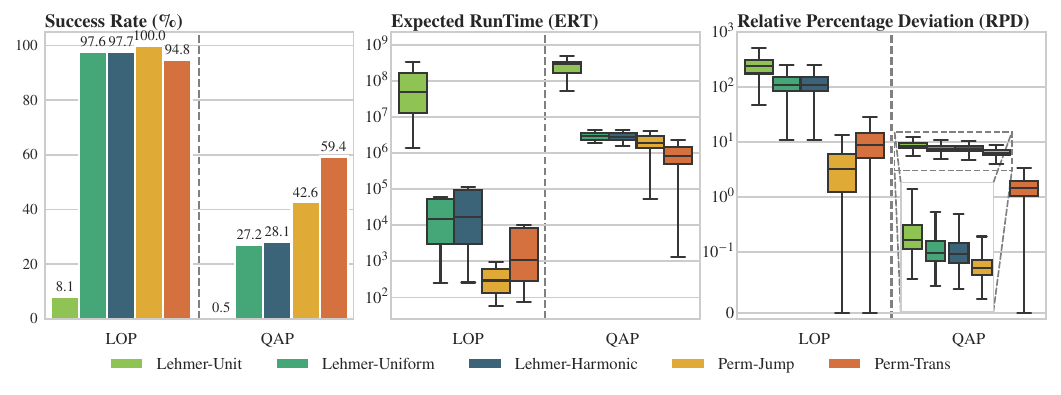}
    \caption{Summary of the results obtained in the experiments on LOP and QAP.}
    \label{fig:exp_lop_qap}
\end{figure*}

We conduct an experimental investigation with a twofold objective.
First, we validate the theoretical runtime analysis on benchmark functions by including the \textit{Harmonic mutation} operator in the comparison, as it has recently become popular in runtime analysis \citep{Doerr2018,FischerLWFOGA23}.
Second, we evaluate the algorithms on real-world instances of the LOP and QAP that, unlike benchmark functions, are NP-hard problems characterized by fitness landscapes that exhibit high multimodality and irregular patterns of neutrality.

The included Harmonic mutation works on the domain $[i]$ by choosing a step size $j\in [1..i-1]$ with probability proportional to $1/j$, and the direction is chosen uniformly at random. Hence, this operator is a compromise between the very local unit mutation and the completely uniform one.
Moreover, since RLS search cannot escape local optima, we limit the empirical analysis to $(1+1)$-EAs.
In total, we empirically evaluate 
six
algorithms: three operating in the Lehmer space---using harmonic (\textit{Lehmer-Harmonic}), uniform (\textit{Lehmer-Uniform}) and unit (\textit{Lehmer-Unit}) mutation---and 
three
operating in the permutation space---employing the standard jump
(\textit{Perm-Jump}), 
transposition
(\textit{Perm-Trans}),
and adjacent swap (\textit{Perm-AdjSwap})
mutation.

To validate and complement the theoretical runtime results on benchmark functions, we considered instances of all theoretical benchmark functions with $n$ ranging from $50$ to $350$. For each algorithm-instance pair, $1000$ independent runs are conducted, except for those shown in the third graph of Fig.~\ref{fig:comparison_all}, where $100$ runs are performed. Each run reports the number of evaluations required to reach the optimum.
Average results are shown in Fig.~\ref{fig:comparison_all}, with equivalent functions grouped together. The graphs show that: (i)~at~least one Lehmer-$\ast$ algorithm always outperforms all the Perm-$\ast$ competitors, and (ii)~Lehmer-Harmonic performs well across all the theoretical benchmarks.

To ensure a more meaningful experimental comparison for LOP/QAP, we slightly modified the implementations of the \mbox{$(1+1)$-EAs} so that iterations in which the mutation does not alter the current solution are not counted toward the evaluation budget---a common aspect of algorithm engineering \cite{PintoDoerrPPSN18,YeEtAlIEEETEC22}.
For the sake of space, we also omit Perm-AdjSwap because its performance is not competitive with the others.
For the LOP, we consider its minimization variant (equivalent to the original maximization version\footnote{Maximizing upper triangular part of the matrix and minimizing lower triangular part induce identical rankings over the solutions.}) in which, given a matrix $\mathbf{B} = [b_{i j}]_{n \times n}$, the objective is to minimize 
\mbox{$f(\sigma) = \sum_{i>j} b_{\sigma(i),\sigma(j)}$}.
In the QAP, an instance is defined by two input matrices $\mathbf{A} = [a_{i j}]_{n \times n}$ and $\mathbf{B} = [b_{i j}]_{n \times n}$, and the goal is to minimize 
$f(\sigma)=\sum_{i,j} a_{i,j} b_{\sigma(i),\sigma(j)}$.
A total of 20 real-world instances have been selected: 10 from the LOP and 10 from the QAP.
LOP instances come from the well-established IO benchmark suite \citep{marti2022exact}, while QAP instances are from the Skorin-Kapov subset of the QAPLIB benchmark collection \citep{burkard1997qaplib}.

Two experiments were conducted: the first follows a fixed-target setting using small instances of size $n = 10$, obtained by subsampling the selected benchmark instances; the second adopts a fixed-budget setting and is performed directly on the original instances, whose sizes range from 42 to 100.
In both experiments, each algorithm was executed $1000$ times per instance.

In the experiment on instances with $n=10$, we began by performing an exhaustive search to identify the global optima for all instances.
We then executed each algorithm repeatedly with a maximum budget of $1\,000\,000$ evaluations per run. For each run, we recorded whether the optimum was reached, and the corresponding runtime---i.e., the number of evaluations required to reach the optimum, or the maximum budget if the optimum was not attained.

The first graph in Fig. \ref{fig:exp_lop_qap} shows the success rates of each algorithm, with results aggregated by problem.
To assess algorithmic efficiency, we also considered the empirical runtime measure, defined as the average runtime divided by the success rate, following \citet{wang2022iohanalyzer}. This corresponds to the expected number of evaluations required to reach the optimum using a multistart version of the algorithm.
Empirical runtimes, aggregated by problem, are shown in the second graph of Fig. \ref{fig:exp_lop_qap}.

These figures show that Lehmer-Harmonic and Lehmer-Uniform are by far more effective than Lehmer-Unit, both approaching the performance of the Perm-$\ast$ algorithms in terms of both success rate and empirical runtime.
This is particularly evident in the LOP, where their success rates surpass that of Perm-Trans and are very close to that of \mbox{Perm-Jump}.

In the experiment on larger instances, global optima are unknown. Consequently, all runs use the entire evaluation budget, which was set to $1000n$ to keep the computational time manageable. Each run returns the best objective value encountered.
For the sake of aggregation, objective values are converted to relative percentage deviations based on the best value observed for each instance.
Those relative percentage deviations, aggregated by problem, are presented in the third graph of Fig. \ref{fig:exp_lop_qap}. The results align with the observations from the small-instances experiment.
\section{Conclusion and Discussion}\label{sec:concl}
We have studied Lehmer codes, also called inversion vectors, as representations for permutations in EAs. Our main focus was the theoretical runtime analysis of simple EAs using Lehmer codes (``Lehmer-EAs'') and a comparison to existing analyses for the classical representation, in particular the seminal work by \citet{Scharnow2005}. As we show for specific benchmarks, there is a clear correspondence between simple mutations in the Lehmer code and well-known fitness measures like the number of inversions in the classical space. Our runtime analyses are asymptotically tight or even non-asymptotic in most cases and reveal that on most benchmarks, the Lehmer-EAs achieve an expected runtime of $\mathcal{O}(n^2\log n)$ or $\mathcal{O}(n^2)$, which is on par with the algorithms for the classical representation or even better by a factor of $\Theta(\log n)$. An exception is the  
$\mathcal{L}\textsc{-LeadingZeros}$ function, where the expected runtime of the Lehmer-EA using unit mutation is worse by a factor of $\Theta(n)$ than the classical approach. This is due to a random-walk behavior, which can be remedied by the more globally searching uniform mutation.

We supplemented experimental studies of the Lehmer-EAs and compared them to the classical EAs on the theoretical benchmarks and on instances of the linear ordering and quadratic assignment problem. On the theoretical benchmarks, the Lehmer-EA with Harmonic mutation is fastest. While the classical algorithms seem to perform best in general on the empirical benchmarks, the performance of the Lehmer-EAs with Harmonic and Uniform mutation operators are not far behind.
In future work, we will analyze whether further runtime improvements for the Lehmer-EAs are possible with advanced operators like self-adjusting mutations and heavy-tailed mutations  \citep{Doerr2018,DoerrLMN17}
or by considering variants of the Lehmer code definitions which can be more suitable for optimization purposes \citep{malagon2025combinatorial}.

Finally, although the results presented focus on simple evolutionary algorithms, they may have far-reaching implications, offering valuable insights for the development of more sophisticated methods based on Lehmer codes, as both optimization (e.g.\ \citealp{UherLehmerTSP}) and learning (e.g.\ \citealp{severo2025learning}) have recently begun to explore their use for handling permutations.
\paragraph{Acknowledgements.} The third author was supported by the Independent Research Fund Denmark (grant~id 10.46540/2032-00101B). Moreover, the research benefited from discussions at Dagstuhl seminar 25092
``Estimation-of-Distribution Algorithms: Theory and Applications''. 
\bibliography{references}

\clearpage 

\appendix

\section*{Appendix~of~``Theoretical~and~Empirical~Analysis of Lehmer Codes to Search Permutation
Spaces with Evolutionary Algorithms''}

\section{Mathematical Analysis Tools and Proofs}
\label{app:proofs}
We first list all the mathematical tools that are used in this paper.
\begin{theorem}
\label{thm:Variable-Drift}
(Variable Drift~\cite{Johannsen2010,Mitavskiy2009})
Let $\{X_t\}_{t\geq 0}$ be a sequence of non-negative random variables with a finite state space $\mathcal{S}\subseteq \mathbb{R}_0^+$ such that $0\in\mathcal{S}$. Let $s_{\text{min}}:=\min(\mathcal{S}\backslash\{0\})$, let $T:=\inf\{t\geq 0\mid X_t=0\}$, and for $t\geq0$ and $s\in\mathcal{S}$ let $\Delta_t(s):=\mathbb{E}[X_t-X_{t+1}\mid X_t=s]$. If there is an increasing function $h:\mathbb{R}^+\rightarrow\mathbb{R}^+$ such that for all $s\in\mathcal{S}\backslash\{0\}$ and all $t\geq0$, $\Delta_t(s)\geq h(s)$, then,
\begin{displaymath}
    \mathbb{E}[T]\leq \frac{s_{\text{min}}}{h(s_\text{min})}+\mathbb{E}\left[\int_{s_\text{min}}^{X_0}\frac{1}{h(\sigma)}\,\mathrm{d}\sigma\right]
\end{displaymath}
\end{theorem}
We also list the Multiplicative Drift Theorem which is a special case of the Variable Drift Theorem.
\begin{theorem}(Multiplicative Drift~\cite{Doerr2012}) 
\label{thm:Multiplicative-Drift}
    Let $\{X_t\}_{t\geq 0}$ be a sequence of non-negative random variables with a finite state space $\mathcal{S}\subseteq \mathbb{R}_0^+$ such that $0\in\mathcal{S}$. Let $s_{\text{min}}:=\min(\mathcal{S}\backslash\{0\})$, let $T:=\inf\{t\geq 0\mid X_t=0\}$, and for $t\geq0$ and $s\in\mathcal{S}$ let $\Delta_t(s):=\mathbb{E}[X_t-X_{t+1}\mid X_t=s]$. Suppose there exists $\delta>0$ such that for all $s\in\mathcal{S}\backslash\{0\}$ and all $t\geq 0$, $\Delta_t(s)\geq \delta s$. Then,
    \begin{displaymath}
        \mathbb{E}[T]\leq \frac{1+\mathbb{E}[\ln(X_0/s_\text{min})]}{\delta}.
    \end{displaymath}
\end{theorem}
The following theorem is well known as Wald's equation~\cite{Wald1944}
\begin{theorem}
\label{thm:Wald's equation}
    Let $X_1,X_2,\dots$ be nonnegative, independent, identically distributed random variables with distribution $X$. Let $T$ be a stopping time for the sequence. If $T$ and $X$ have bounded expectation, then
    \begin{displaymath}
        \mathbb{E}\left[\sum_{i=1}^T X_i\right]=\mathbb{E}[T]\cdot\mathbb{E}[X].
    \end{displaymath}
\end{theorem}
The following theorem states the result of the expected hitting time of the Coupon Collector's Problem with unequal probabilities, a setting that has been thoroughly investigated in~\citet{Flajolet1992}.
\begin{theorem}
\label{thm:CCP}
    Suppose that there are $N$ different types of coupons, with $p_i$ being the probability that the coupon of type $i$ is issued where $p_1+\cdots+p_N=1$. Then let $T$ be the random variable that denotes the first point in time for which all types of coupons are collected, then
    \begin{displaymath}
        \mathbb{E}[T]=\int_0^\infty\left(1-\prod_{i=1}^N(1-\mathrm{e}^{-p_i x})\right)\mathrm{d}x.
    \end{displaymath}
\end{theorem}
Now, we begin to prove the main results in our paper.
\begin{proof}(of Lemma~\ref{lem:EquivalentFunction})
For a $\sigma\in S_n$, assume that $\textsc{LO}(\sigma)=0$. This implies that $\sigma_1\neq 1$, so $L(\sigma)_n>0$ and $\textsc{LZ}(L(\sigma))=0$. Otherwise, assume that $\textsc{LO}(\sigma)=n$, then clearly $\sigma_i=i$ for all $1\leq i\leq n$ and $\textsc{LZ}(L(\sigma))=n$ as well. Note that $\textsc{LO}(\sigma)$, as well as $\textsc{LZ}(L(\sigma))$, cannot take the value $n-1$. Now, assume that $\textsc{LO}(\sigma)=k,1\leq k\leq n-2$, which means $\sigma_{i}=i$ for all $1\leq i\leq k$ and $\sigma_{k+1}\neq k+1$. This implies that $L(\sigma)_{n-i+1}=0$ for all $1\leq i\leq k$ because $\sigma_j>i$ for all $j>i$, while $L(\sigma)_{n-k}>0$ because $\sigma_{k+1}>k+1$ and there must exist a position $j$ where $j>k+1$ and $\sigma_j=k+1$. Hence, we conclude that for any $\sigma\in S_n$, $\textsc{LO}(\sigma)=\textsc{LZ}(L(\sigma))$. 

Regarding $\textsc{LexVal}$, first note that it is a bijection from $S_n$ to $[0..n!-1]$ by definition. Also note that $\textsc{FacVal}$ is a function from $L_n$ to integers with minimum $0$ and maximum $\sum_{i=2}^n (i-1)(i-1)!=n!-1$. Meanwhile, for any $l,\lambda\in L_n$, $l\neq \lambda$ implies that there exists at least one position $i$ such that $l_i\neq \lambda_i$. Thus, $\textsc{FacVal}(l)\neq\textsc{FacVal}(\lambda)$ because $1\cdot (i-1)!>\sum_{j=2}^{i-1}(j-1)\cdot (j-1)!$. Therefore, based on the fact that $|L_n|=n!$, $\textsc{FacVal}$ is a bijection from $L_n$ to $[0..n!-1]$. 

For any $\sigma,\tau \in S_n$, assume that $\textsc{LexVal}(\sigma)<\textsc{LexVal}(\tau)$, which implies that either $\sigma_1<\tau_1$ or there exists a $k$, $2\leq k\leq n-1$, such that $\sigma_i=\tau_i$ for all $1\leq i\leq k-1$ and $\sigma_k<\tau_k$. In the first case, this implies that $L(\sigma)_n<L(\tau)_n$. Since $1\cdot (n-1)!>\sum_{i=2}^{n-1}(i-1)\cdot(i-1)!$, we know $\textsc{FacVal}(\sigma)<\textsc{FacVal}(\tau)$. Similarly, in the second case, we have $L(\sigma)_{n-i+1}=L(\tau)_{n-i+1}$ for all $1\leq i\leq k-1$ and $L(\sigma)_{n-k+1}<L(\tau)_{n-k+1}$, and $\textsc{FacVal}(\sigma)<\textsc{FacVal}(\tau)$.

Therefore, we conclude that for any $\sigma\in S_n$, $\textsc{LexVal}(\sigma)=\textsc{FacVal}(L(\sigma))$.
\end{proof}
Before proving Lemma~\ref{lem:AdjacentLehmerCode-2}, we state the following helper lemma.
\begin{lemma}
    \label{lem:AdjacentLehmerCode}
For any $\sigma\in S_n$, we have $\sigma_i>\sigma_{i+1}\text{ iff. }L(\sigma)_{n-i+1}>L(\sigma)_{n-i}$ for all $1\leq i\leq n-1$.
\end{lemma}
\begin{proof}(of Lemma~\ref{lem:AdjacentLehmerCode})
Assume that $\sigma_i>\sigma_{i+1}$, then for all $j\in[i+2..n]$, $\sigma_j<\sigma_{i+1}$ implies that $\sigma_j<\sigma_i$. Thus, $L(\sigma)_{n-i+1}$ should be at least as large as $L(\sigma)_{n-i}$. Moreover, we know that $(\sigma_i,\sigma_{i+1})$ is an inversion that will increase $L(\sigma)_{n-i+1}$ by $1$. Therefore, $L(\sigma)_{n-i+1}$ is strictly greater than $L(\sigma)_{n-i}$. To prove the other direction, we assume that there exists $i\in[1..n-1]$ s.t. $L(\sigma)_{n-i+1}>L(\sigma)_{n-i}$ and $\sigma_i<\sigma_{i+1}$, which clearly leads to a contradiction because for every $j>i+1$, $(\sigma_i,\sigma_j)$ being an inversion implies that $(\sigma_{i+1},\sigma_j)$ is also an inversion.
\end{proof}
\begin{proof}(of Lemma~\ref{lem:AdjacentLehmerCode-2})
    Assume that $L(\sigma)_{n-i+1}>L(\sigma)_{n-i}$. By Lemma~\ref{lem:AdjacentLehmerCode}, we know that $\sigma_i>\sigma_{i+1}$. If we swap $\sigma_i$ and $\sigma_{i+1}$, then the inversions $(\sigma_j,\sigma_{k})$ where $j=[1..n]\backslash\{i,i+1\}$ and $k>j$ will not be affected. Hence $L(\tau)_{n-j+1}=L(\sigma)_{n-j+1}$ for all $j\in[1..n]\backslash\{i,i+1\}$. Meanwhile, after swapping $\sigma_i$ and $\sigma_{i+1}$, we have 
    \begin{itemize}
        \item Any inversion $(\sigma_i,\sigma_j)$ and $(\sigma_{i+1},\sigma_j)$ where $j>i+1$ will remain in $\tau$,
        \item The inversion $(\sigma_i,\sigma_{i+1})$ will no longer exist in $\tau$
    \end{itemize}
    Thus, $L(\tau)_{n-i+1}=L(\sigma)_{n-i}$ and $L(\tau)_{n-i}=L(\sigma)_{n-i+1}-1$. Note that $L(\sigma)_{n-i+1}-1\geq 0$ since $L(\sigma)_{n-i+1}>L(\sigma)_{n-i}\geq0$.

    Now, assume that $L(\sigma)_{n-i+1}\leq L(\sigma)_{n-i}$. By Lemma~\ref{lem:AdjacentLehmerCode}, we know that $\sigma_i<\sigma_{i+1}$. Same as before, we know that $L(\tau)_{n-j+1}=L(\sigma)_{n-j+1}$ for all $j\in[1..n]\backslash\{i,i+1\}$. Similarly, all inversions $(\sigma_i,\sigma_j)$ and $(\sigma_{i+1},\sigma_j)$ where $j>i+1$ will remain in $\tau$. And $(\sigma_{i+1},\sigma_i)$ will be the new inversion in $\tau$. Therefore $L(\tau)_{n-i+1}=L(\sigma)_{n-i}+1$ and $L(\tau)_{n-i}=L(\sigma)_{n-i+1}$.
\end{proof}
\begin{proof}(of Theorem~\ref{thm:RLS-uniform-OneMax})
The optimization process of $\textsc{RLS}$ on $\textsc{FacVal}$ is exactly the same as that of $\mathcal{L}\textsc{-OneMax}$, thus we only need to derive bounds for the latter one. Let T be the random variable denoting the optimization time. We first derive an upper bound of $\mathbb{E}[T]$. Let $x^{(t)},t=0,1,\dots$ be the random variable that denotes the search point at time $t$, and $y^{(t)},t=1,2,\dots$ the random variable that denotes the solution after applying the uniform step operator to a certain position of $x^{(t-1)}$. Then, $x^{(t)}=y^{(t)}$ if $\mathcal{L}\textsc{-OneMax}(y^{(t)})\leq\mathcal{L}\textsc{-OneMax}(x^{(t-1)})$ and $x^{(t)}=x^{(t-1)}$ otherwise. To give the upper bound, we pessimistically assume that $x^{(0)}_i\neq 0$ for all $i\in[2..n]$. Since $\textsc{RLS}$ flips only one position at a time, once a position is flipped to zero at time $t$, then $x^{(t)}=y^{(t)}$, and this position will remain zero in $x^{(t+1)},x^{(t+2)},\dots$. We call time $t$ a \textit{correction time} if a position of $x^{(t-1)}$ is flipped to zero (and will be accepted), or a position of $x^{(t-1)}$ which is already zero is flipped to $1$ (and will not be accepted). Let $N$ be the random variable that denotes the number of correction times during the optimization process, and let $t_k,k=1,2,\dots,N$ be the random variable that denotes the $k$-th correction time, and $t_0=0$. Then, let $T_k:=t_k-t_{k-1},k=1,2,\dots,N$. Clearly, $T=t_n=\sum_{k=1}^N T_k$. We can see that for every $k\in[1..N]$, $T_k$ is a geometrically distributed random variable with parameter $p$ where
\begin{displaymath}
    p=\frac{1}{n-1}\left(\sum_{i=2}^n \frac{1}{i-1}\right)=\frac{H_{n-1}}{n-1}.
\end{displaymath}
Let $\tilde{T}\sim \text{Geo}(H_{n-1}/(n-1))$, then $T_1,T_2,\dots,T_N$ are nonnegative, independent and identically distributed random variables with distribution $\tilde{T}$. Applying  Wald's equation (see Theorem~\ref{thm:Wald's equation}), we have,
\begin{displaymath}
    \mathbb{E}[T]=\mathbb{E}[N]\mathbb{E}[\tilde{T}]=\frac{n-1}{H_{n-1}}\mathbb{E}[N].
\end{displaymath}
While at a correction time, a non-zero position $i$ is flipped to zero can be viewed as the first time collecting the $i$-th coupon, a zero position $j$ is flipped to one can be viewed as collecting an already collected $j$-th coupon. Meanwhile, suppose that time $t$ is a correction time, then the probability that $i$-th position is flipped by $\textsc{RLS}$ ($i\in[2..n]$) is $1/((i-1)H_{n-1})$. Thus, $N$ is also a stopping time for the Coupon Collector's problem with unequal probabilities $p_i=1/(iH_{n-1}),i\in[1..n-1]$. By Theorem~\ref{thm:CCP}, we know 
\begin{displaymath}
    \mathbb{E}[N]=\int_0^\infty\left(1-\prod_{i=1}^{n-1}\left(1-\exp\left[-\frac{x}{i H_{n-1}}\right]\right)\right)\mathrm{d}x.
\end{displaymath}
Let 
\begin{align*}
    I_n:&=\int_0^\infty\left(1-\prod_{i=1}^n \left(1-\exp\left[-\frac{x}{iH_n}\right]\right)\right)\,\mathrm{d}x\\
    &=H_n\int_{0}^\infty \left(1-\prod_{i=1}^n\left(1-\exp[-x/i]\right)\right)\,\mathrm{d}x.
\end{align*}
We know $1-\exp[-x/i]\geq 1-\exp[-x/n]$ for all $i\in[1..n]$, thus,
\begin{align*}
    I_n&\leq H_n \int_0^\infty \left(1-\left(1-\exp\left[-x/n\right]\right)^n\right)\,\mathrm{d}x\\
    &=nH_n\int_0^\infty (1-(1-\exp[-x])^n)\,\mathrm{d}x.
\end{align*}
Using the inequality $(1-1/x)^{x-1}\geq \mathrm{e}^{-x}$, we further have 
\begin{align*}
    I_n&\leq nH_n\int_0^\infty \left(1-\exp\left[-\frac{n}{\mathrm{e}^x-1}\right]\right)\,\mathrm{d}x\\
    &=nH_n\Bigg(\int_0^{\ln(n+1)}\left(1-\exp\left[-\frac{n}{\mathrm{e}^x-1}\right]\right)\,\mathrm{d}x\\
    &+\int_{\ln(n+1)}^\infty \left(1-\exp\left[-\frac{n}{\mathrm{e}^x-1}\right]\right)\,\mathrm{d}x\Bigg).
\end{align*}
One can see that for $0\leq x\leq \ln(n+1)$,
\begin{displaymath}
    1-1/\mathrm{e}\leq1-\exp\left[-\frac{n}{\mathrm{e}^x-1}\right]\leq 1.
\end{displaymath}
Thus,
\begin{align*}
    \int_{0}^{\ln(n+1)}\left(1-\exp\left[-\frac{n}{\mathrm{e}^x-1}\right]\right)\,\mathrm{d}x&\leq \int_0^{\ln(n+1)}1\,\mathrm{d}x\\
    &=\ln(n+1).
\end{align*}
And, we know that $1-\exp[-n/(\mathrm{e}^x-1)]\leq n/(\mathrm{e}^x-1)$ by using inequality $\mathrm{e}^{-x}\geq 1-x$. So,
\begin{displaymath}
    \int_{\ln(n+1)}^\infty \left(1-\exp\left[-\frac{n}{\mathrm{e}^x-1}\right]\right)\,\mathrm{d}x\leq \int_{\ln(n+1)}^\infty \frac{n}{\mathrm{e}^x-1}\,\mathrm{d}x.
\end{displaymath}
We let $t\gets 1-\mathrm{e}^{-x}$, then
\begin{displaymath}
    \int_{\ln(n+1)}^\infty \frac{n}{\mathrm{e}^x-1}\,\mathrm{d}x=n\int_{n/(n+1)}^1\frac{1}{t}\,\mathrm{d}t=\ln\left(1+\frac{1}{n}\right)^n.
\end{displaymath}
Since $(1+1/n)^n\leq \mathrm{e}$, we have $\ln(1+1/n)^n\leq 1$. Thus, $I_n\leq nH_n(\ln(n+1)-1)$. Hence,
$\mathbb{E}[N]=I_{n-1}\leq (n-1)H_{n-1}(1+\ln n)$, and $\mathbb{E}[T]\leq (n-1)^2\ln n+(n-1)^2$.

Before we start to derive the lower bound, we show that $I_n$ which is defined above can be bounded from below by $cnH_n(\ln n+\ln(1-c)+\gamma$ for any $0\leq c\leq 1-1/n$. We know that
\begin{displaymath}
    I_n=H_n\int_{0}^\infty \left(1-\prod_{i=1}^n\left(1-\exp[-x/i]\right)\right)\,\mathrm{d}x.
\end{displaymath}
Since the term $0\leq 1-\exp[-x/i]\leq 1-\exp[-x/i_\text{min}]\leq 1$ for every $i\in[1..n]$, and $(1-1/n)^n\leq \mathrm{e}$,
\begin{align*}
    \prod_{i=1}^n(1-\exp[-x/i])&\leq \prod_{i=cn+1}^n(1-\exp[-x/i])\\&\leq \left(1-\exp\left[-\frac{x}{cn}\right]\right)^{(1-c)n}\\
    &\leq \exp\left[-\frac{(1-c)n}{\exp\left[\frac{x}{cn}\right]}\right].
\end{align*}
Thus,
\begin{align*}
    I_n&\geq H_n\int_0^\infty \left(1-\exp\left[-\frac{(1-c)n}{\exp\left[\frac{x}{cn}\right]}\right]\right)\,\mathrm{d}x\\
    &=cnH_n\int_0^\infty \left(1-\exp\left[-\frac{(1-c)n}{\mathrm{e}^x}\right]\right)\,\mathrm{d}x.
\end{align*}
Let $m\gets(1-c)n$ and $v\gets m/\mathrm{e}^x$, then
\begin{displaymath}
    \int_0^\infty \left(1-\exp\left[-\frac{(1-c)n}{\mathrm{e}^x}\right]\right)\,\mathrm{d}x=\int_0^m \frac{1-\mathrm{e}^{-v}}{v}\,\mathrm{d}v.
\end{displaymath}
By Taylor's series,
\begin{align*}
    \int_0^m \frac{1-\mathrm{e}^{-v}}{v}\,\mathrm{d}v=-\sum_{n=1}^\infty \frac{(-1)^nm^n}{nn!}.
\end{align*}
And (see Equation 5.1.11 in~\citealt{handbook}),
\begin{displaymath}
    -\sum_{n=1}^\infty \frac{(-1)^nm^n}{nn!}=\gamma+\ln m+E_1(m),
\end{displaymath}
where $E_1(m)=\int_{m}^\infty(\mathrm{e}^{-t}/t)\,\mathrm{d}t>0$. Thus, $I_n\geq cnH_n(\ln n+\ln(1-c)+\gamma)$.

Now we start to give the lower bound of the expected optimization time. We first consider the initial search point $x^{(0)}$. Let $X$ be the random variable that denotes the number of non-zero positions. Then,
\begin{displaymath}
    X=\sum_{i=2}^n \mathbf{1}_{\{x^{(0)}_i\neq 0\}}.
\end{displaymath}
We see that each random variable $\mathbf{1}_{\{x^{(0)}_i\neq 0\}}\sim$\\$\text{Bernoulli}(1-1/i)$. Thus $\mathbb{E}[X]=n-H_n$. Hence, by Chernoff's multiplicative bound, we have 
\begin{displaymath}
    \Pr[X\leq (1-\delta)\mathbb{E}[X]]\leq\exp\left[-\frac{\delta^2 \mathbb{E}[X]}{2}\right]. 
\end{displaymath}
Take $\delta=(\ln n)/\sqrt{n}$, and let $\theta=(1-\delta)\mathbb{E}[X]$, then $\Pr[X\leq \theta]=o(1)$. Let $A$ be the random variable denoting the event that $X\geq \theta$, then, $\Pr[A]=1-o(1)$. And,
\begin{align*}
    \mathbb{E}[T]&=\mathbb{E}[T\mid A]\Pr[A]+\mathbb{E}[T\mid \bar{A}]\Pr[\bar{A}]\\
    &\geq (1-o(1))\mathbb{E}[T\mid A].
\end{align*}
Assume that $A$ occurs. Since we are deriving the lower bound, we can pessimistically assume that the positions $2,3,\dots \theta+1$ of $x^{(0)}$ are not zero, while the positions $\theta+2,\dots,n$ are all zero. Now, different from the upper bound case, we call time $t$ a \textit{correction time} if a position $i\in[2..\theta+1]$ of $x^{(t-1)}$ is flipped to zero, or a position $i\in[2..\theta+1]$ of $x^{(t-1)}$ that is already zero is flipped to $1$. The definition of $t_k,T_k, N$ remain the same. $T_k$ is still a geometrically distributed random variable, but with parameter $p$ where
\begin{displaymath}
    p=\frac{1}{n-1}\left(\sum_{i=2}^{\theta+1}\frac{1}{i-1}\right)=\frac{H_\theta}{n-1} 
\end{displaymath}
in this case. Meanwhile, $N$ becomes the stopping time of a Coupon Collector's problem with $\theta$ coupons and unequal probabilities $p_i=1/(i H_\theta), i=[1..\theta]$. Thus, $\mathbb{E}[N]=I_\theta$. Therefore, take $c=1-1/\ln n$ and have
\begin{align*}
    \mathbb{E}[T\mid A]&\geq \frac{n-1}{H_\theta}\cdot c\,\theta H_\theta(\ln n+\ln (1-c)+\gamma)\\
    &\geq(n-1)(1-1/\ln n)\theta (\ln n-\ln\ln n),
\end{align*}
and $\theta=(1-\delta)(n-H_n)\geq (1-\delta)(n-\ln n-1)$. Hence, $\mathbb{E}[T\mid A]=(1-o(1))(n-1)^2\ln n$, and $\mathbb{E}[T]=(n-1)^2\ln n-o(n^2\log n)$. 
\end{proof}

Before we dive into the $\mathcal{L}\textsc{-LeadingZeros}$ fitness function, we first give the following lemma which reveals that the first time when positions $n,n-1,\dots,i+1$ are all zero, the value of position $i$ is uniformly distributed over $\{0,1,\dots,i-1\}$.
\begin{lemma}
\label{lem:RLS-position-i}
    Consider $\textsc{RLS}$ with either uniform step operator or $\pm1$ step operator and with either uniform probability vector or proportional probability vector on $\mathcal{L}\textsc{-LeadingZeros}$. Let $x^{(t)}_i,i\in[2..n]$ be the random variable that denotes the value of position $i$ at time $t$. Let $\tau_i:=\min\{t\geq 0\mid x^{(t)}_n=x^{(t)}_{n-1}=\cdots=x^{(t)}_{i+1}=0\}$. Then  $x^{(\tau_i)}_i\sim\text{U}(\{0,1,\dots,i{-}1\})$.
\end{lemma}
\begin{proof}(of Lemma~\ref{lem:RLS-position-i})

    Let $x^{(t)}=(x^{(t)}_n,x^{(t)}_{n-1},\dots,x^{(t)}_2),t=0,1,\dots$ be the random variable that denotes the search point at time $t$, and $y^{(t)}=(y^{(t)}_n,y^{(t)}_{n-1},\dots,y^{(t)}_2),t=1,2,\dots$ the random variable that denotes the solution after applying the step operator to a certain position of $x^{(t-1)}$. Then $x^{(t)}=y^{(t)}$ if 
    \begin{displaymath}
        \mathcal{L}\textsc{-LeadingZeros}(y^{(t)})\geq\mathcal{L}\textsc{-LeadingZeros}(x^{(t-1)}),
    \end{displaymath} and $x^{(t)}=x^{(t-1)}$ otherwise. We first prove that $x^{(\tau_i)}_i\mid\{\tau_i=k\}\sim\text{U}(\{0,1,\dots,i{-}1\})$ holds for all $k=0,1,2,\dots$. Obviously, it holds for $k=0$ because $x^{(0)}_i$ is chosen uniformly at random from $\{0,1,\dots,i-1\}$. Now, suppose that $\tau_i=k,k=1,2,\dots$, then we prove it by induction on time $t$. Obviously, $x^{(0)}_i\mid\{\tau_i=k\}\sim\text{U}(\{0,1,\dots,i-1\})$. Now, assume that $x^{(t-1)}_{i}\mid\{\tau_i=k\}\sim\text{U}(\{0,1,\dots,i-1\})$ where $1\leq t\leq k-1$. Then, for $x^{(t)}_i\mid \{\tau_i=k\}$, there are two cases:
    \begin{itemize}
        \item position $l\neq i$ is chosen: Then, whether the new solution is accepted (i.e., $x^{(t)}=y^{(t)}$) or not (i.e., $x^{(t)}=x^{(t-1)}$), $x^{(t)}_i=x^{(t-1)}_i$ because position $i$ is not touched.
        \item position $i$ is chosen: First, note that whatever value $y^{(t)}_i$ takes, $y^{(t)}$ will be accepted (i.e., $x^{(t)}=y^{(t)}$) since we know that when $t<\tau_i$, there exists at least one position $z\in[i{+}1..n]$ such that $x^{(t)}_z\neq 0$.
        \begin{itemize}
            \item Suppose that the uniform step operator is applied, then  
            \begin{align*}
                &\Pr[y^{(t)}_i=j\mid \tau_i=k]\\
                &=\sum_{m\neq j}\Pr[y^{(t)}_i=j\mid \tau_i=k,x^{(t-1)}_i=m]\\
                &\cdot\Pr[x^{(t-1)}_i=m\mid \tau_i=k]\\
                &=\sum_{m\neq j}\frac{1}{i-1}\frac{1}{i}=\frac{1}{i}.
            \end{align*}
            \item Suppose that the $\pm 1$ step operator is applied, then, for $j=1,2,\dots,i-2$,
            \begin{align*}
                &\Pr[y^{(t)}_i=j\mid \tau_i=k]\\
                &=\Pr[y_i^{(t)}=j\mid \tau_i=k,x^{(t-1)}_i=j-1]\\
                &\cdot\Pr[x^{(t-1)}_i=j-1\mid \tau_i=k]+\\
                &\Pr[y_i^{(t)}=j\mid \tau_i=k,x^{(t-1)}_i=j+1]\\
                &\cdot\Pr[x^{(t-1)}_i=j+1\mid \tau_i=k]\\
                &=\frac{1}{2}\frac{1}{i}+\frac{1}{2}\frac{1}{i}=\frac{1}{i}.
            \end{align*}
            For corner cases $j=0$ and $j=i-1$, this holds as well.
        \end{itemize}
        Thus, $y^{(t)}_i\mid \{\tau_i=k\}\sim\text{U}(\{0,1,\dots,i{-}1\})$, and so does $x^{(t)}_i\mid \{\tau_i=k\}$. 

        Through induction, we know that $x^{(\tau_i-1)}_i\sim\text{U}(\{0,1,\dots,i{-}1\})$. Moreover, for time $\tau_i$, there must exist exactly one position $z\in[i{+}1..n]$ that is flipped from a non-zero value to zero, which means that position $i$ cannot be touched. Thus, $x^{(\tau_i)}_i\mid \{\tau_i=k\}=x^{(\tau_i-1)}_i\mid\{\tau_i=k\}\sim\text{U}(\{0,1,\dots,i{-}1\})$. Thus,
        \begin{align*}
            \Pr[x^{(\tau_i)}_i=j]&=\sum_{k=0}^\infty \Pr[x^{(\tau_i)}_i=j\mid \tau_i=k]\Pr[\tau_i=k]\\
            &=\frac{1}{i}\sum_{k=0}^\infty \Pr[\tau_i=k]=\frac{1}{i}.
        \end{align*}
        Therefore, $x^{(\tau_i)}_i\sim\text{U}(0,1,\dots,i{-}1)$.
    \end{itemize}
\end{proof}

\begin{proof}(of Theorem~\ref{thm:RLS-uniform-LeadingZeros})
Let $T$ denote the random variable representing the optimization time, and let $T_i,i\in[2..n]$ be the random variables that denote the optimization time after the first time that positions $n,n-1,\dots,i+1$ all are zero. Then, $T=T_n$. By Wald's equation~ (Theorem~\ref{thm:Wald's equation}), we know that $\mathbb{E}[T_i-T_{i-1}]=(n-1)\mathbb{E}[t_i]$ where $n-1$ is the expected waiting time before position $i$ is selected and flipped, and $t_i$ is the hitting time of the following random walk: 

Let $\{X_t\}_{t\geq 0}$ be a sequence of random variables with state space $\mathcal{S}=\{0,1,\dots,i-1\}$. Assume that $\Pr[X_{t+1}=k\mid X_t=j]=1/(i-1)$ for all $j\in\{0,1,\dots,i-1\}$ and $k\in\{0,1,\dots,i-1\}\backslash\{j\}$. Then, $t_i=\min\{t\geq 0\mid X_t=0\}$. 

By Lemma~\ref{lem:RLS-position-i}, we know that $X_0\sim\text{U}(\{0,1,\dots,i{-}1\})$. Let $\bar{T}_k$ be the expected hitting time of the described random walk starting from state $k$, and $\bar{T}=\sum_{k=0}^{i-1}\bar{T}_k$. Then, $\bar{T}_0=0$ and for any $k=1,2,\dots,i-1$,
\begin{displaymath}
    \bar{T}_k=1+\frac{1}{i-1}\sum_{j\neq k}\bar{T}_j=1+\frac{1}{i-1}\left(\bar{T}-\bar{T}_k\right)=\frac{i-1}{i}+\frac{\bar{T}}{i}.
\end{displaymath}
Summing all $k$ up, then we have,
\begin{displaymath}
    \bar{T}=\sum_{k=0}^{i-1}\bar{T}_k=(i-1)\left(\frac{i-1}{i}+\frac{\bar{T}}{i}\right)=(i-1)^2.
\end{displaymath}
And since $X_0\sim\text{U}(\{0,1,\dots,i-1\})$,
\begin{displaymath}
    \mathbb{E}[t_i]=\frac{1}{i}\left(\bar{T}_0+\bar{T}_1+\cdots+\bar{T}_{i-1}\right)=\frac{\bar{T}}{i}=\frac{(i-1)^2}{i}.
\end{displaymath}
Therefore,
\begin{align*}
    \mathbb{E}[T]&=\sum_{i=2}^n (n-1)\mathbb{E}[t_i]=\sum_{i=2}^n\frac{(n-1)(i-1)^2}{i}\\
    &=(n-1)\sum_{i=2}^n \left((i-1)-\frac{i-1}{i}\right)\\
    &=n^3/2-2n^2+nH_n+3n/2-H_n.
\end{align*}
\end{proof}
\begin{proof}(of Theorem~\ref{thm:RLS-uniform-proportional-OneMax}) This proof mainly follows the approach used in~\citet{Doerr2018}. Same as before, we only consider the expected optimization time on $\mathcal{L}\textsc{-OneMax}$ since the time on $\textsc{FacVal}$ is the same. We define a stochastic process $\{X_t\}_{t\geq 0}$ where,
    \begin{displaymath}
        X_t:=\sum_{i=2}^n\mathbf{1}_{\{x^{(t)}_i\neq 0\}}.
    \end{displaymath}
    The state space is $\mathcal{S}=\{0,1,\dots,n-1\}$. First, we can see that for any $t\geq 0$, either $X_{t+1}=X_{t}$, or $X_{t+1}=X_t-1$. That is because $\textsc{RLS}$ only flips one position at a time, and for a position $i$, if $x^{(\tau)}_i=0$ for some time $\tau$, then for any $t>\tau$, $x^{(t)}_i=0$ as well. Let $T_k:=\min\{t\geq0\mid X_t=k\}$ where $k=0,1,\dots,X_0$. Thus, $T=T_0$ and $T_{X_0}=0$. Hence,
    \begin{align*}
        \mathbb{E}[T\mid X_0]&=\mathbb{E}\left[\sum_{k=1}^{X_0} T_{k-1}-T_k\mid X_0\right]\\
        &=\sum_{k=1}^{X_0}\mathbb{E}[T_{k-1}-T_k\mid X_0].
    \end{align*}
    We can see that $(T_{k-1}-T_k)\mid X_0$ is the waiting time that exactly one of the $k$ non-zero positions is flipped to zero. This happens with probability
    \begin{displaymath}
        \sum_{x^{(t)}_i>0}\frac{2(i-1)}{n(n-1)}\frac{1}{i-1}=\frac{2k}{n(n-1)}.
    \end{displaymath}
    So $(T_{k-1}-T_{k})\mid X_0$ is geometrically distributed with parameter $2k/(n(n-1))$. Therefore, 
    \begin{displaymath}
        \mathbb{E}[T\mid X_0]=\sum_{i=1}^{X_0}\frac{n(n-1)}{2k}=n(n-1)H_{X_0}/2.
    \end{displaymath}
    By the law of total expectation, $\mathbb{E}[T]=\mathbb{E}[\mathbb{E}[T\mid X_0]]=n(n-1)\mathbb{E}[H_{X_0}]/2$. We know that $H_n\geq \ln(n+1)$, thus, $\mathbb{E}[H_{X_0}]\geq\mathbb{E}[\ln(X_0+1)]$. Since $\mathbb{E}\left[\mathbf{1}_{\{x^{(0)}_i\neq 0\}}\right]=1-1/i$, we know that $\mu:=\mathbb{E}[X_0]=n-H_n$. By Chernoff's bound, we know, $\Pr[X_0\leq \mu/2]\leq \exp\left[-\mu/8\right]$. Thus,
    \begin{displaymath}
        \mathbb{E}[\ln(X_0+1)]\geq \mathbb{E}[\ln(X_0+1)\mid X_0>\mu/2]\Pr[X_0>\mu/2],
    \end{displaymath}
    where $\Pr[X_0>\mu/2]>(1-\exp[-\mu/8])$ and $\mathbb{E}[\ln(X_0+1)\mid X_0>\mu/2]>\ln \mu-\ln 2=\ln n-\mathcal{O}(1)$. Therefore, $\mathbb{E}[H_{X_0}]\geq\ln n-\mathcal{O}(1)$. Meanwhile, by Jensen's inequality, $\mathbb{E}[H_{X_0}]\leq \mathbb{E}[1+\ln(X_0+1)]\leq 1+\ln(1+\mathbb{E}[X_0])=1+\ln (1+n-H_n)=\ln n+\mathcal{O}(1)$. Thus, $\mathbb{E}[T]=n(n-1)(\ln(n)+\Theta(1))/2$.
\end{proof}

\begin{proof}(of Theorem~\ref{thm:RLS-uniform-proportional-LeadingZeros})
    Note that Lemma~\ref{lem:RLS-position-i} still holds in the case. Thus, this proof will be the same as the proof of Theorem~\ref{thm:RLS-uniform-LeadingZeros}, except for $T_i-T_{i-1}=\mathbb{E}[t_i]\cdot n(n-1)/(2(i-1))$ because now the waiting time for position $i$ to be flipped is $n(n-1)/(2(i-1))$. Therefore,
    \begin{align*}
        \mathbb{E}[T]&=\sum_{i=2}^n \frac{n(n-1)}{2(i-1)}\cdot \frac{(i-1)^2}{i}\\
        &=n^3/2-n^2 H_n/2-n^2/2+nH_n/2.
    \end{align*}
\end{proof}

\begin{proof}(of Theorem~\ref{thm:RLS-unit-OneMax}) This proof mainly follows the approach used in~\citet{Doerr2018}. We define a stochastic process $\{X_t\}_{t\geq 0}$ where
\begin{displaymath}
    X_t:=\sum_{i=2}^n (2^{x^{(t)}_i}-1).
\end{displaymath}
with state space $\mathcal{S}$ where $0\in\mathcal{S}$ and $s_\text{min}=\min(\mathcal{S}\backslash\{0\})=2^1-1=1$. Note that $X_{t+1}\leq X_t$ and $X_{t+1}<X_t$ if and only if some position $i$ with $x^{(t)}_i>0$ is flipped down, this happens with probability $1/(2(n-1))$ and $X_t-X_{t+1}$ is $2^{x^{(t)}_i}-2^{x^{(t)}_i-1}$. Thus,
\begin{align*}
    &\mathbb{E}[X_t-X_{t+1}\mid X_t=s]\\
    &=\sum_{i\in[2..n],x^{(t)}_i>0}\frac{1}{2(n-1)}\left(1-\frac{1}{2}\right)2^{x^{(t)}_i}\\
    &\geq \sum_{i=2}^n \frac{1}{4(n-1)}\left(2^{x^{(t)}_i}-1\right)=\frac{s}{4(n-1)}
\end{align*}
Meanwhile $\ln(\mathbb{E}[X_0])\leq \ln( 2^1+2^2+\cdots+2^{n-1})\leq n\ln 2$. By the  Multiplicative Drift Theorem (Theorem~\ref{thm:Multiplicative-Drift})  where $\delta=1/(4(n-1))$,
\begin{displaymath}
    \mathbb{E}[T]\leq\frac{1+\mathbb{E}[\ln X_0]}{1/(4(n-1))}.
\end{displaymath}
By Jensen's inequality $\mathbb{E}[\ln X_0]\leq \ln \mathbb{E}[X_0]$. Therefore $\mathbb{E}[T]\leq (4\ln 2)n(n-1)+4(n-1)=\mathcal{O}(n^2)$. 

Now, we consider the lower bound. Let $T_n$ the expected optimization time of the first position (index $n$). Clearly $\mathbb{E}[T]\geq \mathbb{E}[T_n]$. Also note that if a position is chosen for mutation and the unit step operator tries $+1$, then the step must be rejected and the value stays the same. Otherwise, if $-1$ is tried and the current value is greater than zero, then the step will be accepted. Therefore, by Wald's equation (Theorem~\ref{thm:Wald's equation}), $\mathbb{E}[T_n]=(n-1)\tau$,  where $(n-1)$ is the expected waiting time before the first position is flipped and $\tau$ is the stopping time of the following random walk.

Let $\{X_t\}_{t\geq 0}$ be a sequence of random variables with state space $\mathcal{S}=\{0,1,\dots,n-1\}$. Assume that $\Pr[X_{t+1}=k\mid X_t=k]=\Pr[X_{t+1}=k-1\mid X_t=k]=1/2$ for all $k=1,2,\dots,n-1$, and $\Pr[X_{t+1}=0\mid X_t=0]=1$. Then $\tau=\min\{t\geq 0\mid X_t=0\}$.

Since the first position(position $n$) of the initial search point is sampled uniformly at random over $\{0,1,\dots,n-1\}$, we have $X_0\sim\text{U}(\{0,1,\dots,n-1\})$. Let $S_k$ the expected hitting time of the above random walk starting from state $k,k\in[0..n-1]$. Then, $S_0=0$ and 
\begin{displaymath}
    \tau=\frac{1}{n}\left(S_0+S_1+\cdots+S_{n-1}\right).
\end{displaymath}
We know that for $k\in[1..n-1]$,
\begin{displaymath}
    S_{k}=1+\frac{1}{2}S_k+\frac{1}{2}S_{k-1}=2+S_{k-1}=2(k-1).
\end{displaymath}
Therefore $\tau=(1/n)(2(n-1)n/2)=n-1$ and $\mathbb{E}[T]\geq \mathbb{E}[T_n]=(n-1)^2$.
\end{proof}

Consider a fair random walk with an absorbing state at $0$ and a reflecting boundary at $i-1$. Such a random walk is also related to the well-known ``Gambler's Ruin Problem'' (see, e.g.,~\citealp{feller1}
). The expected hitting time of such a process has been extensively studied (see, e.g.,~\citealp{Aldous2019}).
\begin{theorem}
\label{thm:Random-Walk}
    Let $\{X_t\}_{t\geq 0}$ be a sequence of random variables with state space $\mathcal{S}=\{0,1,\dots,i-1\}$. Assume that $\Pr[X_{t+1}=k-1\mid X_t=k]=\Pr[X_{t+1}=k+1\mid X_t=k]=1/2$ for all $k=1,2,\dots,i-2$, and $\Pr[X_{t+1}=i-1\mid X_{t}=i-1]=\Pr[X_{t+1}=i-2\mid X_t=i-1]=1/2$ while $\Pr[X_{t+1}=0\mid X_t=0]=\Pr[X_{t+1}=1\mid X_t=0]=1/2$. Let $T:=\min\{t\geq 0\mid X_t=0\}$, then $\mathbb{E}[T]=\mathbb{E}[X_0(2i-1-X_0)]$.
\end{theorem}
\begin{proof}(of Theorem~\ref{thm:RLS-unit-LeadingZeros}) 
The proof is the same as the proof of Theorem~\ref{thm:RLS-uniform-LeadingZeros}, except for $\mathbb{E}[T_i-T_{i-1}]=(n-1)\mathbb{E}[t_i]$ where $t_i$ is now the hitting time of the random walk described in Theorem~\ref{thm:Random-Walk}. By Lemma~\ref{lem:RLS-position-i}, we know that $X_0\sim\text{U}(\{0,1,\dots,i-1\})$ for the position $i$. Thus,
\begin{displaymath}
    \mathbb{E}[t_i]=\frac{1}{i}\sum_{k=0}^{i-1} k(2i-1-k)=\frac{(i-1)(2i-1)}{3}.
\end{displaymath}
Thus,
\begin{align*}
    \mathbb{E}[T]&=\sum_{i=2}^n (n-1)\mathbb{E}[t_i]\\
    &=\sum_{i=2}^n (n-1)\frac{(i-1)(2i-1)}{3}\\
    &=\frac{2}{9}n^4-\frac{7}{18}n^3+\frac{1}{9}n^2+\frac{1}{18}n.
\end{align*}
\end{proof}

\begin{proof}(of Theorem~\ref{thm:linear-function-lower-bound})
    W.\,l.\,o.\,g.\ we can assume that $n$ is even. Let $T$ be the random variable that denotes the whole optimization time while $T_{n/2}$ be the random variable that denotes the first time when the first $n/2$ positions, i.e., positions from $n$ down to $n/2+1$, are zeros. Clearly we have $\mathbb{E}[T]\geq \mathbb{E}[T_{n/2}]$ since the first $n/2$ positions must have been set to zero before hitting the optimum where all positions are zero. The following proof mainly follows~\citet{DoerrPohl} that uses a similar argument as for the general lower bound in \citet{DJWoneone}.
    
    Consider the initial search point $x^{(0)}$. The probability that a position $i\in[n/2+1..n]$ is initialized as a nonzero value is $(i-1)/i\geq 1-2/n$. Let $X:=\sum_{i=n/2+1}^n\mathbf{1}_{\{x^{(0)}_i\neq 0\}}$, then we know $\mathbb{E}[X]\geq \frac{n}{2}(1-\frac{2}{n})=\frac{n}{2}-1$. By Chernoff's bound, we have,
    \begin{align*}
        \Pr\left[X\leq\frac{n}{6}\right]&=\Pr\left[X\leq \left(1-\frac{2n-6}{3n-6}\right)\left(\frac{n}{2}-1\right)\right]\\
        &\leq \Pr\left[\left(1-\frac{2n-6}{3n-6}\right)\mathbb{E}[X]\right]\\
        &\leq \exp\left[-\left(\frac{2n-6}{3n-6}\right)^2\frac{\mathbb{E}[X]}{2}\right]\\
        &\leq \exp\left[-\left(\frac{2n-6}{3n-6}\right)^2 \frac{n/2-1}{2}\right],
    \end{align*}
    which is smaller than $\exp[-n/36]$ for all $n\geq 5$. Thus, with overwhelming probability, at least $n/6$ positions out of the first $n/2$ positions are initialized as nonzero. In each iteration, the probability that a nonzero position $i\in[n/2+1..n]$ is flipped to zero is $1/((n-1)(i-1))\leq 2/(n(n-1))$. Let $\epsilon>0$. The probability that a nonzero position $i\in[n/2+1..n]$ is never flipped to zero during the first $\tilde{T}:=(1-\epsilon)((n-1)n/2-1)\ln n$ iterations can be bounded from below by 
    \begin{align*}
        &\left(1-\frac{1}{(n-1)(i-1)}\right)^{\tilde{T}}\geq\left(1-\frac{2}{n(n-1)}\right)^{\tilde{T}}\\
        &=\left(\left(1-\frac{2}{n(n-1)}\right)^{n(n-1)/2-1}\right)^{(1-\epsilon)\ln n}\\
        &\geq \exp[-(1-\epsilon)\ln n]=1/n^{1-\epsilon},
    \end{align*}
    where the last inequality uses the fact that $(1-1/n)^{n-1}\geq 1/\mathrm{e}$ for all $n\geq 1$. Therefore, suppose that there are more than $n/6$ positions out of the first $n/2$ positions which are initialized as nonzero, then the probability that, $n/6$ such positions, denoted by $i_1,i_2,\dots,i_{n/6}$, are flipped to zero at least once during the first $\tilde{T}$ iterations can be bounded from above by 
    \begin{align*}
        &\prod_{k=1}^{n/6}\left(1-\left(1-\frac{1}{(n-1)(i_k-1)}\right)^{\tilde{T}}\right)\leq\left(1-\frac{1}{n^{1-\epsilon}}\right)^{n/6}\\
        &\leq \exp[-n^{\epsilon}/6],
    \end{align*}
    where the last inequality uses the fact that $(1-1/n)^n\leq 1/\mathrm{e}$ for all $n\geq 1$. Let $E$ be the event that at least $n/6$ positions out of the first $n/2$ positions are initialized as nonzero and there exists at least one position $i\in[n/2+1..n]$, which is initialized as nonzero and never flipped to zero during the first $\tilde{T}$ iterations. Then,
    \begin{align*}
        \mathbb{E}[T]&\geq \mathbb{E}[T_{n/2}]\geq \mathbb{E}[T_{n/2}\mid E]\Pr[E]\geq \tilde{T}\Pr[E]\\
        &\geq \tilde{T}(1-\exp[-n/36])(1-\exp[-n^{\epsilon}/6])\\
        &=\Omega(n^2\log n).
    \end{align*}
\end{proof}

\begin{proof}(of Theorem~\ref{thm:(1+1)-EA-uniform-OneMax})
Let $T$ be that random variable that denotes the optimization time. Let $x^{(t)}=(x^{(t)}_n,x^{(t)}_{n-1},\dots,x^{(t)}_2),t=0,1,\dots$ be the random variable that denotes the search point at time $t$. Consider the stochastic process $\{X_t\}_{t\geq0}$ where $X_t:=\sum_{i=2}^n x^{(t)}_i$. To simplify the reading, we let $x_i=x^{(t)}_i$. Let $E$ be the event that only one position is flipped at time $t+1$, then $\Pr[E]=(1-1/(n-1))^{n-2}\geq 1/\mathrm{e}$. And
\begin{align*}
    \mathbb{E}[X_t-X_{t+1}\mid E]&=\frac{1}{n-1}\sum_{i=2}^n \sum_{j=1}^{x_i}\frac{j}{i-1}\\
    &=\frac{1}{n-1}\sum_{i=2}^n\frac{x_i(x_i+1)}{2(i-1)}.
\end{align*}
Using the Cauchy–Schwarz inequality, we know
\begin{align*}
    \sum_{i=2}^n \frac{x_i(x_i+1)}{i-1}&\geq \frac{\left(\sum_{i=2}^n \sqrt{x_i(x_i+1)}\right)^2}{\sum_{i=2}^n (i-1)}\\
    &=\frac{\left(\sum_{i=2}^n \sqrt{x_i(x_i+1)}\right)^2}{n(n-1)/2}\\
    &\geq \frac{\left(\sum_{i=2}^n x_i\right)^2}{n(n-1)/2}=\frac{2(X_t)^2}{n(n-1)}.
\end{align*}
Meanwhile,
\begin{displaymath}
    \sum_{i=2}^n \frac{x_i(x_i+1)}{i-1}\geq \sum_{i=2}^n \frac{x_i(x_i+1)}{n-1}\geq \sum_{i=2}^n \frac{2x_i}{n-1}=\frac{2X_t}{n-1}.
\end{displaymath}
Define $h(s): \mathbb{R}^+\rightarrow\mathbb{R}^+$ as
\begin{displaymath}
    h(s):=\begin{cases}
    s/(\mathrm{e}(n-1)^2) & s<n\\
    s^2/(\mathrm{e}n(n-1)^2) & s\geq n
    \end{cases}
\end{displaymath}
Note that $h(s)$ monotonically increases when $s>0$. For any $X_t=s$ where $s\neq 0$ and $\Pr[X_t=s]>0$, we have
\begin{align*}
    &\mathbb{E}[X_t-X_{t+1}\mid X_t=s]\\
    &\geq \mathbb{E}[X_t-X_{t+1}\mid X_t=s,E]\Pr[E\mid X_{t}=s].
\end{align*}
Therefore, $\mathbb{E}[X_t-X_{t+1}\mid X_t=s]\geq h(s)$. Applying the Variable Drift Theorem (Theorem~\ref{thm:Variable-Drift}) with $s_\text{min}=1$, we have
\begin{displaymath}
    \mathbb{E}[T]\leq \frac{1}{h(1)}+\mathbb{E}\left[\int_1^{X_0}\frac{1}{h(\sigma)}\,\mathrm{d}\sigma\right].
\end{displaymath}
We know $1/h(1)=\mathrm{e}(n-1)^2$, while
\begin{align*}
    \int_1^{X_0}\frac{1}{h(\sigma)}\,\mathrm{d}\sigma&=\int_1^{n}\frac{1}{h(\sigma)}\,\mathrm{d}\sigma+\int_n^{X_0}\frac{1}{h(\sigma)}\,\mathrm{d}\sigma\\
    &=\mathrm{e}(n-1)^2\ln n+\mathrm{e}n(n-1)^2\left(\frac{1}{n}-\frac{1}{X_0}\right)\\
    &\leq \mathrm{e}(n-1)^2\ln n+\mathrm{e}(n-1)^2-2\mathrm{e}(n-1),
\end{align*}
where the last inequality uses the fact that $X_0\leq \sum_{i=2}^n(i-1)=n(n-1)/2$. Hence, $\mathbb{E}[T]\leq \mathrm{e}(n-1)^2\ln n+2\mathrm{e}(n-1)^2-2\mathrm{e}(n-1)$.
\end{proof}

\begin{proof}(of Theorem~\ref{thm:(1+1)-EA-uniform-FacVal})
The construction of the potential function is based on the technique in~\cite{WITT2013}, which was generalized by \cite{DoerrPohl} for the analysis of the multi-valued search space $[r]^n$. Let $T$ be the random variable that denotes the optimization time. Let $x^{(t)}=(x^{(t)}_n,\dots,x^{(t)}_2),t=0,1,\dots$ be the random variable that denotes the search point at time $t$. Let $I:=\{2,3,\dots,n\}$. Consider the stochastic process $\{X_t\}_{t\geq 0}$ which is defined as $X_t:=\sum_{i\in I}g_i\ln(x^{(t)}_i+1)$, where $g_i:=(1+\lambda)^{i-2}$ for some constant $c>0$ and $\lambda:=(c\ln n)/(n-1)$. First, we can see that
\begin{displaymath}
    X_0\leq \sum_{i\in I}g_i\ln n=\frac{(1+\lambda)^{n-1}{-}1}{\lambda}\ln n\leq \frac{(1+\lambda)^{n-1}\ln n}{\lambda}. 
\end{displaymath}
Since $(1+\lambda)^{n-1}\leq \exp[c\ln n]=n^c$ and $(\ln n)/\lambda=(n-1)/c$, we have $X_0\leq n^{c+1}/c$ and $\ln(X_0)\leq (c+1)\ln n-\ln c$.

Consider a certain time $t\geq 0$, we define
\begin{itemize}
    \item $l(i)=x^{(t)}_i-x^{(t+1)}_i,i\in I$,
    \item $d(i)=\ln(x^{(t)}_i+1)-\ln(x^{(t+1)}_i+1),i\in I$,
    \item $I_D=\{i\in I\mid l(i)>0\}$ (the set of all $down$-$flipped$ positions),
    \item $I_U=\{i\in I\mid l(i)<0\}$ (the set of all $up$-$flipped$ positions),
    \item $I_0=\{i\in I\mid l(i)=0\}$ (the set of all $non$-$flipped$ positions),
    \item $L(i)=\{n,n-1,\dots,i\}$,
    \item $R(i)=\{i-1,i-2\dots,2\}$,
    \item $\Delta_t=X_t-X_{t+1}$.
\end{itemize}
Note that $l(i)$ and $d(i)$ always have the same sign and both $(I_D,I_U,I_0)$ and $(L(i),R(i))$ are partitions of $I$. Clearly,
\begin{align*}
    &\Delta_t=\sum_{j\in I}g_j\cdot d(j)=\sum_{j\in I_D}g_j\cdot d(j)+\sum_{j\in I_U}g_j\cdot d(j)\\
    &=\sum_{j\in I_D}g_j\cdot d(j)+\sum_{j\in I_U\cap L(i)}g_j\cdot d(j)+\sum_{j\in I_U\cap R(i)}g_j\cdot d(j).
\end{align*}
Let 
\begin{displaymath}
    \Delta_L(i)=\sum_{j\in I_D}g_j\cdot d(j)+\sum_{j\in I_U\cap L(i)}g_j\cdot d(j).
\end{displaymath}
And 
\begin{displaymath}
    \Delta_R(i)=\sum_{j\in I_U\cap R(i)}g_j\cdot d(j).
\end{displaymath}
Obviously $\Delta_t=\Delta_L(i)+\Delta_R(i)$. Define $A_i,i\in I$ the event that occurs when the following three conditions are satisfied: 
\begin{enumerate}
    \item $I_D\neq\emptyset$,
    \item $i=\max I_D$,
    \item $I_U\cap L(i)=\emptyset$.
\end{enumerate}
Note that due to the property of $\textsc{FacVal}$, $x^{(t+1)}\neq x^{(t)}$ if and only if $f(x^{(t+1)})< f(x^{(t)})$. Thus, $x^{(t+1)}\neq x^{(t)}$ if and only if some positions are flipped and the leftmost flipped position is down-flipped. Let $A$ be the event that $(x^{(t)}\neq x^{(t+1)})$, and we can see that $A=\bigcup_{i\in I}A_i$, and $A_i\cap A_j=\emptyset$ for all $i,j\in I,i\neq j$. Therefore,
\begin{align*}
    &\mathbb{E}[\Delta_t]=\mathbb{E}[\Delta_t\mid A]\Pr[A]+\mathbb{E}[\Delta_t\mid\bar{A}]\Pr[\bar{A}]\\
    &=\mathbb{E}[\Delta_t\mid A]\Pr[A]=\sum_{\substack{i\in I \\x^{(t)}_i>0}}\mathbb{E}[\Delta_t\mid A_i]\Pr[A_i]\\
    &=\sum_{\substack{i\in I \\x^{(t)}_i>0}}\mathbb{E}[\Delta_L(i)\mid A_i]\Pr[A_i]+\mathbb{E}[\Delta_R(i)\mid A_i]\Pr[A_i].
\end{align*}
We know that for any $i\in I$, 
\begin{displaymath}
    \Pr[A_i]=\left(1-\frac{1}{n-1}\right)^{n-i}\frac{1}{n-1}\frac{x^{(t)}_i}{i-1}.
\end{displaymath}
Note that 
\begin{align*}
    &0\geq \mathbb{E}[\Delta_R(i)\mid A_i]\geq -\frac{1}{n-1}\sum_{j=2}^{i-1}g_j(\ln n-\ln 1)\\
    &=-\frac{\ln n}{n-1}\sum_{j=2}^{i-1}(1+\lambda)^{j-2}=-\frac{\ln n}{n-1}\frac{(1+\lambda)^{i-2}-1}{\lambda}\\
    &=-\frac{\ln n}{n-1}\frac{g_i-1}{(c\ln n)/(n-1)}\geq -\frac{g_i}{c}.
\end{align*}
Meanwhile, $\Pr[A_i]\leq x^{(t)}_i/((n-1)(i-1))$. Therefore, 
\begin{displaymath}
    \mathbb{E}[\Delta_R(i)\mid A_i]\Pr[A_i]\geq-\frac{g_i\cdot x^{(t)}_i}{c(n-1)(i-1)}.
\end{displaymath}
We know that 
\begin{align*}
    &\mathbb{E}[\Delta_L(i)\mid A_i]\\
    &=\mathbb{E}\left[\sum_{j\in I_D}g_j\cdot d(j)+\sum_{j\in I_U\cap L(i)} g_j\cdot d(j)\mid A_i\right]\\
    &=\mathbb{E}\left[\sum_{j\in I_D}g_j\cdot d(j)\mid A_i\right]\geq g_i\mathbb{E}[d(i)\mid A_i].
\end{align*}
Let $z=x^{(t)}_i>0$, then for all $z\geq 1$,
\begin{align*}
    \mathbb{E}[d(i)\mid A_i]&=\frac{1}{z}\sum_{j=0}^{z-1}(\ln(z+1)-\ln(j+1))\\
    &=\frac{1}{z}\left(\ln(z+1)^z-\ln (z!)\right)\geq \ln 2.
\end{align*}
And $\Pr[A_i]\geq (1-1/(n-1))^{n-2}x^{(t)}_i/((n-1)(i-1))\geq x^{(t)}_i/(\mathrm{e}(n-1)(i-1))$, hence,
\begin{displaymath}
    \mathbb{E}[\Delta_L(i)\mid A_i]\Pr[A_i]\geq \frac{(\ln 2)g_i\cdot x^{(t)}_i}{\mathrm{e}(n-1)(i-1)}.
\end{displaymath}
Therefore,
\begin{align*}
    \mathbb{E}[\Delta_t]&\geq \sum_{\substack{i\in I\\x^{(t)}_i>0}}\left(\frac{\ln 2}{\mathrm{e}}-\frac{1}{c}\right)\frac{g_i \cdot x^{(t)}_i}{(n-1)(i-1)}\\
    &\geq \left(\frac{\ln 2}{\mathrm{e}}-\frac{1}{c}\right)\frac{1}{(n-1)^2}\sum_{\substack{i\in I\\x^{(t)}_i>0}}g_i \cdot x^{(t)}_i\\
    &\geq \left(\frac{\ln 2}{\mathrm{e}}-\frac{1}{c}\right)\frac{1}{(n-1)^2}\sum_{\substack{i\in I\\x^{(t)}_i>0}}g_i \cdot \ln(x^{(t)}_i+1)\\
    &= \left(\frac{\ln 2}{\mathrm{e}}-\frac{1}{c}\right)\frac{X_t}{(n-1)^2}=:\delta X_t.
\end{align*}
Applying the Multiplicative Drift Theorem (Theorem~\ref{thm:Multiplicative-Drift}) where $s_\text{min}=g_2\ln 2=\ln 2$, we know that 
\begin{align*}
    \mathbb{E}[T]&\leq \frac{1+\mathbb{E}[\ln(X_0/s_\text{min})]}{\delta}\\
    &\leq \frac{1+(c+1)\ln n-\ln c-\ln\ln 2}{\delta}.
\end{align*}
One can see that when $c>\mathrm{e}/(\ln 2)$, $\mathbb{E}[T]\leq 69.2(n-1)^2\ln n+\mathcal{O}(n^2)$.
\end{proof}

\begin{proof}(of Theorem~\ref{thm:(1+1)-EA-uniform-NVal})
    The lower bound is $\Omega(n^2\log n)$, as given in~\citet{DoerrPohl}. The proof for the upper bound is basically the same as the proof of Theorem~\ref{thm:(1+1)-EA-uniform-FacVal}. Let $T$ be the random variable that denotes the optimization time, and let $I=\{1,2,\dots,n\}$. Consider the stochastic process $\{X_t\}_{t\geq0}$ which is defined as $X_t=\sum_{i\in I}g_i\ln(x^{(t)}_i+1)$ where $g_i:=(1+\lambda)^{i-1}$ for some constant $c>0$. Let $R(i)=\{i-1,i-2,\dots,1\}$. The definitions of $\lambda,l(i),d(i),I_D,I_U,I_0,L(i),\Delta_t,\Delta_L(i),\Delta_R(i),A_i$ and $A$ remain the same. First note that $\ln(X_0)\leq(c+1)\ln n-\ln c$.

    From the proof of Theorem~\ref{thm:(1+1)-EA-uniform-FacVal}, we know that
    \begin{displaymath}
        \mathbb{E}[\Delta_t]{=}\sum_{\substack{i\in I\\ x^{(t)}_i>0}}\mathbb{E}[\Delta_L(i)\mid A_i]\Pr[A_i]+\mathbb{E}[\Delta_R(i)\mid A_i]\Pr[A_i].
    \end{displaymath}
    In this case,
    \begin{displaymath}
        \Pr[A_i]=\left(1-\frac{1}{n}\right)^{n-i}\frac{1}{n}\frac{x^{(t)}_i}{n-1}.
    \end{displaymath}
    And, $x^{(t)}_i/(\mathrm{e}n(n-1))\leq\Pr[A_i]\leq x^{(t)}_i/(n(n-1))$. We also know that for any $i\in I$ where $x^{(t)}_i>0$,
    \begin{displaymath}
        \mathbb{E}[\Delta_R(i)\mid A_i]\Pr[A_i]\geq -\frac{g_i\cdot x^{(t)}_i}{cn(n-1)}, 
    \end{displaymath}
    while 
    \begin{displaymath}
        \mathbb{E}[\Delta_L(i)\mid A_i]\Pr[A_i]\geq \frac{(\ln 2)g_i\cdot x^{(t)}_i}{\mathrm{e}n(n-1)}.
    \end{displaymath}
    Therefore,
    \begin{align*}
        \mathbb{E}[\Delta_t]&\geq \sum_{\substack{i\in I\\x^{(t)}_i>0}}\left(\frac{\ln 2}{\mathrm{e}}-\frac{1}{c}\right)\frac{g_i\cdot x^{(t)}_i}{n(n-1)}\\
        &\geq \sum_{\substack{i\in I\\x^{(t)}_i>0}}\left(\frac{\ln 2}{\mathrm{e}}-\frac{1}{c}\right)\frac{g_i\cdot \ln(x^{(t)}_i+1)}{n(n-1)}\\
        &=\left(\frac{\ln 2}{\mathrm{e}}-\frac{1}{c}\right)\frac{X_t}{n(n-1)}.
    \end{align*}
    Using Multiplicative Drift Theorem where $s_\text{min}=\ln 2$, and choosing $c$ properly, we have $\mathbb{E}[T]=\mathcal{O}(n^2\log n)$.
\end{proof}

Consider $(1+1)$-EA on $\mathcal{L}\textsc{-LeadingZeros}$, we have the following Lemma which is similar to Lemma~\ref{lem:RLS-position-i}.
\begin{lemma}
\label{lem:(1+1)EA-position-i}
Consider $(1+1)$-EA with either uniform step operator or $\pm1$ step operator on $\mathcal{L}\textsc{-LeadingZeros}$. Let $x^{(t)}_i,i\in[2..n]$ be the random variable that denotes the value of position $i$ at time $t$. Let $\tau_i:=\min\{t\geq 0\mid x_n^{(t)}=\cdots= x_{i+1}^{(t)}=0\}$. Then $x^{(\tau_i)}_i\sim\text{U}(\{0,1,\dots,i-1\})$.
\end{lemma}
\begin{proof}(of Lemma~\ref{lem:(1+1)EA-position-i})
    The proof idea is similar to that of Lemma~\ref{lem:RLS-position-i}. Let $x^{(t)}=(x^{(t)}_n,x^{(t)}_{n-1},\dots,x^{(t)}_2),t=0,1,\dots$ be the random variable that denotes the search point at time $t$, and $y^{(t)}=(y^{(t)}_n,y^{(t)}_{n-1},\dots,y^{(t)}_2)$ be the random variable that denotes the solution after applying the step operator to a certain position of $x^{(t-1)}$. If $\tau_i=0$, then clearly $x^{(\tau_i)}_i\mid \{\tau_i=0\}\sim\text{U}(\{0,1,\dots,i-1\})$. Now, suppose that $\tau_i=k,k\geq1$. We prove $x^{(\tau_i)}_i\mid \{\tau_i=k\}$ by induction on time $t$. Obviously, $x^{(0)}_i\mid \{\tau_i=k\}\sim\text{U}(\{0,1,\dots,i-1\})$ since we initialize each position uniformly at random. Assume that $x^{(t-1)}_i\mid\{\tau_i=k\}\sim\text{U}(0,1,\dots,i-1)$ holds for all $1\leq t\leq k-1$. Then, for $x^{(t)}_i\mid\{\tau_i=k\}$, there are two cases,
    \begin{itemize}
        \item With probability $1/(n-1)$ position $i$ is flipped (i.e., $y^{(t)}_i\neq x^{(t)}_i$). From the proof of Lemma~\ref{lem:RLS-position-i}, we know that $y^{(t)}_i\mid\{\tau_i=k\}\sim\text{U}(\{0,1,\dots,i-1\})$. Since there are still nonzero positions before position $i$, $x^{(t)}_i=y^{(t)}_i$ for always. Thus, $x^{(t)}_i\mid\{\tau_i=k\}\sim\text{U}(0,1,\dots,i-1)$.
        \item With probability $1-1/(n-1)$ position $i$ is not flipped, then $x^{(t)}_i\mid\{\tau_i=k\}=x^{(t-1)}_i\mid\{\tau_i=k\}\sim\text{U}(0,1,\dots,i-1)$
    \end{itemize}
    Therefore, we know that $x^{(\tau_i-1)}_i\mid \{\tau_i=k\}\sim\text{U}(\{0,1,\dots,i-1\})$. At time $\tau_i$, all nonzero positions before $i$ are flipped to zero, resulting in a fitness improvement. Thus, if position $i$ is not flipped, it will remain the original value which is uniformly distributed. Otherwise, if position $i$ is flipped, the new solution will always adopt the flipped value which is still uniformly distributed regardless of the specific step operator used. Hence, we conclude that $x^{(\tau_i)}_i\sim\text{U}(0,1,\dots,i-1)$. 
\end{proof}

\begin{proof}(of Theorem~\ref{thm:(1+1)-EA-uniform-LeadingZeros})
Let $T$ be the random variable that denotes the optimization time, and let $T_i,i\in[2..n]$ be the random variable that denotes the optimization time after the first time that positions $n,n-1,\dots,i+1$ are all zero, and $T_1:=0$. Clearly, $T=T_n=\sum_{i=2}^n (T_i-T_{i-1})$. We can see that $T_i-T_{i-1}$ is the optimization time for position $i$ to be flipped to zero, and can be written as the sum of random variables, namely $T_i-T_{i-1}=\sum_{k=1}^{t_i}X_k$ where $X_k$ is the time between the $k{-}1$-th and $k$-th time that position $i$ is flipped and the solution is accepted, and when $k=t_i$, position $i$ is flipped to zero and the solution is accepted. Suppose that positions $n,n-1,\dots,i+1$ are all zero and the nonzero position $i$ is flipped, then the new solution can be accepted if and only if positions $n,n-1,\dots,i+1$ are not flipped. Thus, $X_k\sim\text{Geo}(p)$ where $p=(1-\frac{1}{n-1})^{n-i}\frac{1}{n-1}$. Therefore, by Wald's equation (Theorem~\ref{thm:Wald's equation}), $\mathbb{E}[T_i-T_{i-1}]=(n-1)(1-1/(n-1))^{i-n}\mathbb{E}[t_i]$. By Lemma~\ref{lem:(1+1)EA-position-i}, we know that the first time positions $n,n-1,\dots,i+1$ are all zero, the value of position $i$ is uniformly distributed over $\{0,1,\dots,i-1\}$. Therefore, $t_i$ is the same as in the proof of Theorem~\ref{thm:RLS-uniform-LeadingZeros}, and $\mathbb{E}[t_i]=(i-1)^2/i$. Hence,
\begin{align*}
    \mathbb{E}[T]&=\sum_{i=2}^n \frac{n-1}{\left(1-\frac{1}{n-1}\right)^{n-i}}\frac{(i-1)^2}{i}\\
    &=\frac{n-1}{(1-1/(n-1))^{n-1}}\sum_{i=1}^{n-1}\frac{i^2}{i+1}\left(1-\frac{1}{n-1}\right)^i.
\end{align*}

\allowdisplaybreaks
Let $N:=n-1$. Then 
\begin{align*}
    &\frac{n-1}{\left(1-\frac{1}{n-1}\right)^{n-1}}=\frac{N}{\left(1-\frac{1}{N}\right)^N}\\
    &=N\cdot \exp\left[-N\ln\left(1-\frac{1}{N}\right)\right]\\
    &=N\cdot\exp\left[-N\left(-\frac{1}{N}-\frac{1}{2N^2}-\frac{1}{3N^3}-\mathcal{O}\left(\frac{1}{N^4}\right)\right)\right]\\
    &=N\cdot\exp\left[1+\frac{1}{2N}+\frac{1}{3N^2}+\mathcal{O}\left(\frac{1}{N^3}\right)\right]\\
    &=N\cdot\mathrm{e}\cdot\left(1+\frac{1}{2N}+\frac{11}{24N^2}+\mathcal{O}\left(\frac{1}{N^3}\right)\right)\\
    &=\mathrm{e}N+\frac{\mathrm{e}}{2}+\frac{11\mathrm{e}}{24N}+\mathcal{O}\left(\frac{1}{N^2}\right).
\end{align*}
Meanwhile, let $q:=1-1/N$. Then
\begin{displaymath}
    \sum_{i=1}^{n-1}\frac{i^2}{i+1}\left(1-\frac{1}{n-1}\right)^i=\sum_{i=1}^N (i-1)q^i+\sum_{i=1}^N \frac{q^i}{i+1}.
\end{displaymath}
Firstly, let $S_1=\sum_{i=1}^N(i-1)q^i$. Then 
\begin{align*}
    S_1&=\frac{q^2-Nq^{N+1}+(N-1)q^{N+2}}{(1-q)^2}\\
    &=N^2\left(1-\frac{2}{N}+\frac{1}{N^2}+\left(-2+\frac{3}{N}-\frac{1}{N^2}\right)q^N\right).\\
    &=N^2-2N+1+(-2N^2+3N-1)q^N.
\end{align*}
And
\begin{align*}
    q^N&=\exp\left[N\ln\left(1-\frac{1}{N}\right)\right]\\
    &=\exp\left[N\left(-\frac{1}{N}-\frac{1}{2N^2}-\frac{1}{3N^3}-\mathcal{O}\left(\frac{1}{N^4}\right)\right)\right]\\
    &=\frac{1}{\mathrm{e}}\left(1-\frac{1}{2N}-\frac{5}{24N^2}-\mathcal{O}\left(\frac{1}{N^3}\right)\right).
\end{align*}
Thus,
\begin{displaymath}
    S_1=\left(1-\frac{2}{\mathrm{e}}\right)N^2+\left(\frac{4}{\mathrm{e}}-2\right)N+\left(1-\frac{25}{12\mathrm{e}}\right)+\mathcal{O}\left(\frac{1}{N}\right).
\end{displaymath}
Secondly, let $S_2=\sum_{i=1}^N q^i/(i+1)$. Then
\begin{displaymath}
    0\leq S_2\leq\sum_{i=0}^\infty \frac{q^i}{i+1}=\frac{-\ln(1-q)}{q}=\frac{\ln N}{1-1/N}.
\end{displaymath}
And
\begin{align*}
    \frac{\ln N}{1-1/N}&=\ln N\left(1+\frac{1}{N}+\frac{1}{N^2}\frac{1}{1-1/N}\right)\\
    &=\ln N+\frac{\ln N}{N}+\mathcal{O}\left(\frac{\ln N}{N^2}\right),N\geq 2.
\end{align*}
Thus,
\begin{displaymath}
    \mathbb{E}[T]=(\mathrm{e}-2)(n-1)^3+(3-3\mathrm{e}/2)(n-1)^2+R_n,
\end{displaymath}
where $R_n=\mathcal{O}(n\log n)$.
\end{proof}

\begin{proof}(of Theorem~\ref{thm:(1+1)-EA-unit-lower-bound})
Let $T$ be the random variable that denotes the optimization time, and let $T_n$ be the random variable that denotes the first time that position $n$ becomes zero. We know that for both $\mathcal{L}\textsc{-OneMax}$ and $\textsc{FacVal}$, once the first position(position $n$) becomes zero, it will not be flipped to some other values afterwards. Thus, $T\geq T_n$. We can write $T_n=\sum_{i=1}^N t_i$ where $t_i\sim\text{Geo}(1/(n-1))$ is the waiting time before the first position is flipped. Thus, again by using Wald's equation, we have $\mathbb{E}[T_n]=(n-1)\mathbb{E}[N]$, where $N$ is the stopping time of the following random walk.

Consider the random walk (see proof of Theorem~\ref{thm:RLS-unit-OneMax}) $\{X_t\}_{t\geq 0}$ whose state space is $\mathcal{S}=\{0,1,\dots,n-1\}$, and $X_0\sim\text{U}(\mathcal{S})$. State $0$ is the absorbing state and $\Pr[X_{t+1}=s\mid X_t=s]=\Pr[X_{t+1}=s-1\mid X_{t}=s]=1/2$ holds for all $s\in\mathcal{S}\backslash\{0\}$. 

Let $\bar{T}_i,i\in\mathcal{S}$ denote the hitting time starting from state $i$. Then, $\mathbb{E}[\bar{T}_i]=2+\mathbb{E}[\bar{T}_{i-1}]$ and $\mathbb{E}[\bar{T}_0]=0$. Thus, $\mathbb{E}[N]=\frac{1}{n}\sum_{i=0}^{n-1}\mathbb{E}[\bar{T}_{i}]=n-1$. Thus, $\mathbb{E}[T]\geq\mathbb{E}[T_n]=(n-1)^2$.
\end{proof}

\begin{proof}(of Theorem~\ref{thm:(1+1)-EA-unit-OneMAx})
The proof will be exactly the same as the proof of Theorem 14 in the paper~\cite{Doerr2018} except for the last step. In our case, the expected optimization time is of order at most $\ln(\frac{w^n-w}{w-1})/(\frac{c}{2wn})=\mathcal{O}(n^2)$.
\end{proof}

\begin{proof}(of Theorem~\ref{thm:(1+1)-EA-unit-FacVal}) 
This proof uses the same method as the proof of Theorem~\ref{thm:(1+1)-EA-uniform-FacVal}. Consider a stochastic process $\{X_t\}_{t\geq 0}$ that is defined as $X_t:=\sum_{i=2}^n g_ix^{(t)}_i$ where $g_i=(1+c/(n-1))^{i-2}$ for some constant $c>0$. $l(i),d(i),I_D,I_U,I_0,L(i),R(i),\Delta_t,\Delta_L(i),\Delta_R(i),A_i$ and $A$ are defined in the same way as in the proof of Theorem~\ref{thm:(1+1)-EA-uniform-FacVal}. First note that
\begin{displaymath}
    X_0\leq (n-1)\frac{\left(1+\frac{c}{n-1}\right)^{n-1}-1}{\frac{c}{n-1}}\leq \frac{\mathrm{e}^c}{c}(n-1)^2,
\end{displaymath}
where the last inequality uses the fact that $(1+1/n)^n\leq \mathrm{e}$ for all $n\geq 1$. Meanwhile, for any $i\in I$ where $x^{(t)}_i>0$,
\begin{displaymath}
    \Pr[A_i]=\left(1-\frac{1}{n-1}\right)^{n-i}\frac{1}{n-1}\frac{1}{2}.
\end{displaymath}
Therefore, $\Pr[A_i]\leq 1/(2(n-1))$, and $\Pr[A_i]\geq (1-1/(n-1))^{n-2}/(2(n-1))\geq 1/(2\mathrm{e}(n-1))$. Moreover, $\mathbb{E}[\Delta_L(i)\mid A_i]\geq g_i\cdot 1$. Also we have 
\begin{align*}
    \mathbb{E}[\Delta_R(i)\mid A_i]&\geq -\frac{1}{n-1}\frac{1}{2}\sum_{j=2}^{i-1}\left(1+\frac{c}{n-1}\right)^{j-2}\cdot 1\\
    &=-\frac{1}{2(n-1)}\frac{\left(1+\frac{c}{n-1}\right)^{i-2}-1}{\frac{c}{n-1}}\\
    &\geq -\frac{g_i}{2c}.
\end{align*}

\allowdisplaybreaks
Hence,
\begin{align*}
    &\mathbb{E}[\Delta_t]\\
    &=\sum_{\substack{i\in I\\x^{(t)}_i>0}}\mathbb{E}[\Delta_L(i)\mid A_i]\Pr[A_i]+\mathbb{E}[\Delta_R(i)\mid A_i]\Pr[A_i]\\
    &\geq \sum_{\substack{i\in I\\x^{(t)}_i>0}}\frac{g_i}{2\mathrm{e}(n-1)}-\frac{g_i}{4c(n-1)}\\
    &\geq \left(\frac{1}{2\mathrm{e}}-\frac{1}{4c}\right)\frac{1}{n-1}\sum_{i\in I}g_i\frac{x^{(t)}_i}{n-1}\\
    &=\left(\frac{1}{2\mathrm{e}}-\frac{1}{4c}\right)\frac{s}{(n-1)^2}.
\end{align*}
Thus, by the Multiplicative Drift Theorem (Theorem~\ref{thm:Multiplicative-Drift}) where $s_\text{min}=1$, and we take $c=2$, we have $\mathbb{E}[T]=\mathcal{O}(n^2\log n)$.
\end{proof}

\begin{proof}(of Theorem~\ref{thm:(1+1)-EA-unit-LeadingZeros}) 
    Let $T$ be the random variable that denotes the optimization time. Same as the proof of Theorem~\ref{thm:(1+1)-EA-uniform-LeadingZeros}, we can write 
    \begin{displaymath}
        \mathbb{E}[T]=\sum_{i=2}^n \frac{n-1}{\left(1-\frac{1}{2}\frac{1}{n-1}\right)^{n-i}}\mathbb{E}[t_i].
    \end{displaymath}
    Note that there is a $1/2$ in the term $(1-\frac{1}{2(n-1)})^{n-i}$ because the value for a zero position being flipped down will remain zero. In this certain problem, $t_i$ is the hitting time of the random walk that is described in the proof of Theorem~\ref{thm:RLS-unit-LeadingZeros}, and $\mathbb{E}[t_i]=(i-1)(2i-1)/3$. Let $N=n-1, q=1-1/(2N)$. Then,
    \begin{align*}
        \mathbb{E}[T]&=\sum_{i=2}^n \frac{n-1}{\left(1-\frac{1}{2}\frac{1}{n-1}\right)^{n-i}}\frac{(i-1)(2i-1)}{3}\\
        &=\frac{N}{\left(1-\frac{1}{2N}\right)^N}\frac{1}{3}\sum_{i=1}^N i(2i+1)q^i. \\
        &=\frac{N}{3q^N}\sum_{i=1}^Ni(2i+1)q^i.
    \end{align*}
    Let $S=\sum_{i=1}^N (2i^2+i)q^i$, we want to find a function $f(x),x\in[0..n]$ s.t. $f(0)=0$ and $f(i)-f(i-1)=(2i^2+i)q^i$. Then, $S=f(N)$. Assume that $f(x)=(ax^3+bx^2+cx+d)q^x-d$, we can solve the equation $f(x)-f(x-1)=(2x^2+x)q^x$ and find out that, $a=0,b=\frac{2q}{q-1},c=\frac{q(q-5)}{(q-1)^2},d=\frac{q(q+3)}{(q-1)^3}$. Therefore,
    \begin{displaymath}
        S=\left(\frac{2q}{q-1}N^2+\frac{q(q-5)}{(q-1)^2}N+\frac{q(q+3)}{(q-1)^3}\right)q^{N}-\frac{q(q+3)}{(q-1)^3}.
    \end{displaymath}
    Since we know $q=1-1/(2N)$,
    \begin{align*}
        \mathbb{E}[T]=\frac{N}{3q^N}S=&-\frac{52}{3}N^4+\frac{28}{3}N^3-\frac{1}{3}N^2\\
        &+\frac{N}{3q^N}(32N^3-20N^2+2N).
    \end{align*}
    Meanwhile,
    \begin{align*}
        \frac{N}{3q^N}(32N^3-20N^2+2N&)=\frac{32\sqrt{\mathrm{e}}}{3}N^4-\frac{16\sqrt{\mathrm{e}}}{3}N^3\\
        &+\frac{13\sqrt{\mathrm{e}}}{36}N^2-\frac{\sqrt{\mathrm{e}}}{48}N-\Theta(1).
    \end{align*}
    Thus,
    \begin{align*}
        \mathbb{E}[T]&=\frac{32\sqrt{\mathrm{e}}-52}{3}N^4+\frac{28-16\sqrt{\mathrm{e}}}{3}N^3+\frac{13\sqrt{\mathrm{e}}-12}{36}N^2\\
        &-\frac{\sqrt{\mathrm{e}}}{48}N-\Theta(1).
    \end{align*}
    \end{proof}

\section{Detailed Experimental Results} \label{app:exp}

Tabs.~\ref{tab:exp_small_sr},~\ref{tab:exp_small_ert},~and~\ref{tab:exp_large_rpd} present the detailed instance-by-instance results on LOP and QAP, which were summarized and discussed in Sect.~\ref{sec:exp}.

Tab.~\ref{tab:exp_small_sr} reports the success rates for each algorithm-instance pair, calculated as the proportion of runs in which the algorithm reached a global optimum for the given instance.
Tab.~\ref{tab:exp_small_ert} presents the expected runtime for each algorithm-instance pair, computed as the average number of evaluations required to reach the optimum divided by the success rate.
Tab.~\ref{tab:exp_large_rpd} shows the average relative percentage deviation for each algorithm-instance pair.
In each run, the relative percentage deviation is given by \mbox{$100 \cdot (\mathit{val}-\mathit{best})/\mathit{best}$}, where $\mathit{val}$ is the objective value returned by the run, and $\mathit{best}$ is the best objective value achieved by any algorithm in the suite on that instance.

The last row of each table reports the average rank achieved by each algorithm across the 20 instances.
We also performed Wilcoxon tests using the standard significance level of 0.05, and controlling the false discovery rate with the Benjamini–Hochberg procedure.
Accordingly, the average ranks in the last row of each table are annotated as follows:
marked with the symbol $\blacktriangle$ if the algorithm is the overall best performer,
left unmarked if its performance is not significantly different from the best algorithm,
or marked with the symbol $\triangledown$ if it is significantly worse than the best algorithm.

While not essential to the analyses presented, we provide, for completeness, information about the computing infrastructure on which the experiments were conducted. 
We performed our experiments in a virtualized HPC environment managed by SLURM. Each job was allocated 60 virtual CPU cores and 150 GB of RAM on a node with 64 vCPUs (Intel(R) Xeon(R) Gold 6230 @ 2.10GHz) and 257 GB total RAM, using a Singularity container based on Ubuntu 20.04 LTS. No GPU was used.

Furthermore, given the stochastic nature of the algorithms, and to support reproducibility, the accompanying source code explicitly specifies the random number generator employed and the seeds used for initialization.

\begin{table*}
\centering
\resizebox{\textwidth}{!}{
\begin{tabular}{llrrrrr}
\toprule
Problem & Instance & Lehmer-Harmonic & Lehmer-Uniform & Lehmer-Unit & Perm-Jump & Perm-Trans \\
\midrule
\multirow[t]{10}{*}{LOP} & N-be75eec & 1.000 & 1.000 & 0.415 & 1.000 & 1.000 \\
 & N-be75np & 0.831 & 0.799 & 0.003 & 1.000 & 0.874 \\
 & N-be75oi & 1.000 & 1.000 & 0.057 & 1.000 & 1.000 \\
 & N-be75tot & 1.000 & 1.000 & 0.009 & 1.000 & 1.000 \\
 & N-stabu70 & 0.939 & 0.971 & 0.000 & 1.000 & 1.000 \\
 & N-stabu74 & 1.000 & 1.000 & 0.000 & 1.000 & 1.000 \\
 & N-stabu75 & 1.000 & 0.993 & 0.001 & 1.000 & 0.604 \\
 & N-t59b11xx & 1.000 & 1.000 & 0.277 & 1.000 & 1.000 \\
 & N-t59d11xx & 1.000 & 1.000 & 0.014 & 1.000 & 1.000 \\
 & N-t59f11xx & 1.000 & 1.000 & 0.038 & 1.000 & 1.000 \\
\cline{1-7}
\multirow[t]{10}{*}{QAP} & sko42 & 0.321 & 0.317 & 0.007 & 0.457 & 0.540 \\
 & sko49 & 0.162 & 0.138 & 0.003 & 0.287 & 0.306 \\
 & sko56 & 0.320 & 0.327 & 0.004 & 0.429 & 0.635 \\
 & sko64 & 0.220 & 0.246 & 0.019 & 0.384 & 0.856 \\
 & sko72 & 0.293 & 0.246 & 0.001 & 0.658 & 0.453 \\
 & sko81 & 0.280 & 0.276 & 0.003 & 0.247 & 0.399 \\
 & sko90 & 0.246 & 0.225 & 0.003 & 0.201 & 0.393 \\
 & sko100a & 0.372 & 0.389 & 0.002 & 0.261 & 0.716 \\
 & sko100b & 0.407 & 0.358 & 0.007 & 0.991 & 1.000 \\
 & sko100c & 0.193 & 0.195 & 0.004 & 0.345 & 0.642 \\
\cline{1-7}
\multicolumn{2}{l}{Avg Ranks \& Stat. Comp.} & 2.85 $\triangledown$ & 3.02 $\triangledown$ & 5.00 $\triangledown$ & 2.33 & \textbf{1.80} $\blacktriangle$ \\
\bottomrule
\end{tabular}
}
\caption{Success rates of $(1+1)$-EAs on instances of size $n=10$.}
\label{tab:exp_small_sr}
\end{table*}

\begin{table*}
\centering
\resizebox{\textwidth}{!}{
\begin{tabular}{llrrrrr}
\toprule
Problem & Instance & Lehmer-Harmonic & Lehmer-Uniform & Lehmer-Unit & Perm-Jump & Perm-Trans \\
\midrule
\multirow[t]{10}{*}{LOP} & N-be75eec & 258 & 302 & 1409840 & 117 & 197 \\
 & N-be75np & 352536 & 398058 & 332333448 & 661 & 246726 \\
 & N-be75oi & 8395 & 5979 & 16547064 & 158 & 496 \\
 & N-be75tot & 3905 & 4382 & 110111190 & 301 & 1383 \\
 & N-stabu70 & 115221 & 61897 & $\infty$ & 287 & 10216 \\
 & N-stabu74 & 36304 & 23504 & $\infty$ & 950 & 759 \\
 & N-stabu75 & 117947 & 148406 & 999000135 & 1906 & 985466 \\
 & N-t59b11xx & 321 & 256 & 2610173 & 58 & 75 \\
 & N-t59d11xx & 2734 & 2660 & 70428648 & 121 & 223 \\
 & N-t59f11xx & 24851 & 23130 & 25329470 & 528 & 2686 \\
\cline{1-7}
\multirow[t]{10}{*}{QAP} & sko42 & 2241183 & 2292814 & 141873905 & 1331505 & 959003 \\
 & sko49 & 5491147 & 6521631 & 332513762 & 2661890 & 2339999 \\
 & sko56 & 2285223 & 2228711 & 249017145 & 1518810 & 629533 \\
 & sko64 & 3612478 & 3123795 & 51671042 & 1824563 & 289011 \\
 & sko72 & 2573233 & 3248222 & 999000304 & 652032 & 1261777 \\
 & sko81 & 2905237 & 2925912 & 332339135 & 3144976 & 1549257 \\
 & sko90 & 3219567 & 3639325 & 332333927 & 4038384 & 1620499 \\
 & sko100a & 1987925 & 1887132 & 499000077 & 3024079 & 458401 \\
 & sko100b & 1581188 & 1908472 & 141892972 & 53546 & 1317 \\
 & sko100c & 4322858 & 4245310 & 249070001 & 1965830 & 654611 \\
\cline{1-7}
\multicolumn{2}{l}{Avg Ranks \& Stat. Comp.} & 3.22 $\triangledown$ & 3.33 $\triangledown$ & 5.00 $\triangledown$ & 1.83 & \textbf{1.61} $\blacktriangle$ \\
\bottomrule
\end{tabular}
}
\caption{Expected RunTime of $(1+1)$-EAs on instances of size $n=10$.}
\label{tab:exp_small_ert}
\end{table*}

\begin{table*}
\resizebox{\textwidth}{!}{\centering
\begin{tabular}{llrrrrrr}
\toprule
Problem & Instance & $n$ & Lehmer-Harmonic & Lehmer-Uniform & Lehmer-Unit & Perm-Jump & Perm-Trans \\
\midrule
\multirow[t]{10}{*}{LOP} & N-be75eec & 50 & 150.69 & 150.69 & 331.38 & 8.70 & 18.83 \\
 & N-be75np & 50 & 180.21 & 180.21 & 360.38 & 1.96 & 9.12 \\
 & N-be75oi & 50 & 97.78 & 97.78 & 178.19 & 6.34 & 13.84 \\
 & N-be75tot & 50 & 181.60 & 181.60 & 318.19 & 2.97 & 14.20 \\
 & N-stabu70 & 60 & 81.63 & 81.63 & 169.55 & 3.24 & 10.39 \\
 & N-stabu74 & 60 & 100.45 & 100.45 & 187.28 & 4.47 & 10.38 \\
 & N-stabu75 & 60 & 97.17 & 97.17 & 185.41 & 4.52 & 11.84 \\
 & N-t59b11xx & 44 & 107.36 & 107.36 & 258.69 & 4.55 & 6.56 \\
 & N-t59d11xx & 44 & 92.64 & 92.64 & 258.67 & 9.27 & 11.27 \\
 & N-t59f11xx & 44 & 132.67 & 132.67 & 260.62 & 1.23 & 2.69 \\
\cline{1-8}
\multirow[t]{10}{*}{QAP} & sko42 & 42 & 9.16 & 9.35 & 11.12 & 7.24 & 2.38 \\
 & sko49 & 49 & 8.90 & 8.97 & 10.47 & 7.48 & 2.36 \\
 & sko56 & 56 & 8.55 & 8.71 & 10.10 & 6.93 & 1.86 \\
 & sko64 & 64 & 7.67 & 7.83 & 8.91 & 6.60 & 1.55 \\
 & sko72 & 72 & 7.52 & 7.63 & 8.59 & 6.05 & 1.42 \\
 & sko81 & 81 & 7.06 & 7.22 & 8.16 & 6.18 & 1.22 \\
 & sko90 & 90 & 7.28 & 7.39 & 8.26 & 6.17 & 1.40 \\
 & sko100a & 100 & 6.82 & 6.90 & 7.65 & 5.99 & 1.18 \\
 & sko100b & 100 & 6.53 & 6.70 & 7.51 & 5.88 & 1.05 \\
 & sko100c & 100 & 7.13 & 7.23 & 8.12 & 6.35 & 1.36 \\
\cline{1-8}
\multicolumn{3}{l}{Avg Ranks \& Stat. Comp.} & 3.25 $\triangledown$ & 3.75 $\triangledown$ & 5.00 $\triangledown$ & \textbf{1.50} $\blacktriangle$ & \textbf{1.50} $\blacktriangle$ \\
\bottomrule
\end{tabular}
}
\caption{Average Relative Percentage Deviations of $(1+1)$-EAs on large instances.}
\label{tab:exp_large_rpd}
\end{table*}


\end{document}